%% file: main.tex
\pdfoutput=1

\documentclass{article}
\usepackage{iclr2027_conference,times}
\input{math_commands.tex}

\usepackage[T1]{fontenc}
\usepackage[hidelinks]{hyperref}
\usepackage{url}
\usepackage{amsmath,amssymb,amsthm}
\usepackage{booktabs,longtable,array}
\usepackage{graphicx}
\usepackage{xcolor}
\usepackage{enumitem}
\usepackage{algorithm}
\usepackage{expl3,xparse}
\usepackage[font=small,skip=6pt]{caption}
\graphicspath{{figs/}}

\input{tables/revision_numbers}

\input{tables/gauss_numbers}   %
\input{tables/holdout_numbers} %
\input{tables/synthetic_dense_numbers} %

\newtheorem{theorem}{Theorem}
\newtheorem{lemma}{Lemma}

\newtheorem{definition}{Definition}
\newtheorem{assumption}{Assumption}
\theoremstyle{remark}

\newcommand{\runin}[1]{\par\noindent\textbf{#1}\quad\ignorespaces}

\newcommand{\tr}{\operatorname{tr}}

\newcommand{\GL}{\mathrm{GL}}

\newcommand{\at}{\tilde\alpha}
\newcommand{\keeplines}[1]{\par\vskip 0pt plus #1\baselineskip\penalty-200\vskip 0pt plus -#1\baselineskip}
\makeatletter
\newcommand{\keeptogether}[1]{\par\penalty-100\begingroup\dimen@=#1\baselineskip\relax
  \ifdim\dimen@>\dimexpr\pagegoal-\pagetotal\relax\vfil\break\fi\endgroup}
\makeatother

\title{Variance Fingerprints:\\ Component-Wise Identification\\ of Content Factors by Augmentation Design}

\author{D. Hudson Smith\\ Clemson University\\ \texttt{dane2@clemson.edu}}
\iclrfinalcopy

\newcommand{\theauthors}{the author}
\newcommand{\Theauthors}{The author}
\newcommand{\theauthorsp}{the author's}

\begin{document}
\maketitle
\lhead{Preprint}

\input{sections/00_abstract}

\input{sections/01_introduction}

\input{sections/02_baseline}

\input{sections/03_identification}

\input{sections/04_learning}

\input{sections/05_experiments}

\input{sections/06_discussion}

\label{end:main}

\input{sections/07_ai_use}

\input{sections/08_reproducibility}

\bibliography{references}
\bibliographystyle{iclr2027_conference}
\clearpage
\appendix
\input{appendices/00_related_work}

\input{appendices/01_proofs}

\input{appendices/02_objective}

\input{appendices/03_implementation}

\input{appendices/04_experiments}

\end{document}

%% file: math_commands.tex
\usepackage{amsmath,amsfonts,bm}

\def\eqref#1{equation~\ref{#1}}

\def\1{\bm{1}}

\DeclareMathAlphabet{\mathsfit}{\encodingdefault}{\sfdefault}{m}{sl}
\SetMathAlphabet{\mathsfit}{bold}{\encodingdefault}{\sfdefault}{bx}{n}

\newcommand{\E}{\mathbb{E}}

\newcommand{\R}{\mathbb{R}}

\newcommand{\Var}{\mathrm{Var}}

\newcommand{\Cov}{\mathrm{Cov}}

%% file: tables/revision_numbers.tex
\newcommand{\sThreeSeedsN}{six}
\newcommand{\sThreeSigOn}{six}
\newcommand{\sThreeSigOff}{zero}
\newcommand{\sThreeAssignedMean}{0.50}
\newcommand{\sThreeOffMean}{-0.12}
\newcommand{\sThreeXsPairs}{15}
\newcommand{\sThreeXsOn}{13}
\newcommand{\sThreeXsOff}{10}

\newcommand{\dsGzeroXsPairs}{three of three}
\newcommand{\amortPlugHalf}{0.0749}
\newcommand{\amortDHalf}{0.0122}

\newcommand{\noiseFloorTwo}{4.93}

\newcommand{\dsMainMccSeeds}{0.933, 0.938, 0.933}

\newcommand{\dsMainKnn}{0.858}

%% file: tables/gauss_numbers.tex
\newcommand{\dsOmZeroMcc}{0.803}
\newcommand{\dsOmZeroKnn}{0.88}
\newcommand{\dsOmZeroNull}{above the mixing null in 1 of 3 runs}

\newcommand{\dsPermNull}{inside the mixing null in every run ($\hat p$ 0.59 to 0.98)}
\newcommand{\dsGzeroMcc}{0.452}

\newcommand{\dsGzeroNull}{above the mixing null in 3 of 3 runs ($\hat p<0.01$; null 99th percentile 0.16 to 0.64)}
\newcommand{\sdOmZeroAssigned}{0.14}

%% file: tables/holdout_numbers.tex
\newcommand{\holdoutN}{2,000}
\newcommand{\holdoutCells}{44}
\newcommand{\holdoutMaxGapMcc}{0.008}
\newcommand{\holdoutMaxGapKnn}{0.029}

\newcommand{\trainDspMccTwo}{0.928}
\newcommand{\trainDspMccHalf}{0.932}

\newcommand{\holdoutDspIdentN}{ten}

\newcommand{\trainGaussDspMcc}{0.935}

\newcommand{\holdoutGaussSthreeMcc}{0.390}
\newcommand{\trainGaussSthreeMcc}{0.397}

\newcommand{\denseHalfShare}{0.90}
\newcommand{\denseHalfAgree}{five}

%% file: tables/synthetic_dense_numbers.tex
\newcommand{\synthDenseMccOff}{0.778}
\newcommand{\synthDenseMccTwo}{0.997}
\newcommand{\synthDenseMccHalf}{0.997}
\newcommand{\synthDenseMccOne}{0.998}
\newcommand{\synthDenseMccOnLo}{0.996}
\newcommand{\synthOracleMccHalf}{0.995}
\newcommand{\synthDenseShareLo}{0.97}
\newcommand{\synthDenseSlopeLo}{0.96}
\newcommand{\synthDenseSlopeHi}{1.02}
\newcommand{\synthDenseAgree}{15}
\newcommand{\synthDenseAgreeN}{15}

%% file: sections/00_abstract.tex
\begin{abstract}
Augmentation-based self-supervised learning can recover the content shared across views without separating its underlying factors into individual coordinates. We ask whether augmentations can be designed to make those factors identifiable. Our approach allows content factors to fluctuate across augmented views while preserving their mean for each source instance. Each instance receives a fixed, recorded set of augmentation strengths that controls how much its factors fluctuate. When factors respond differently to changes in these strengths, their patterns of variation provide information that distinguishes them. Under suitable decoder and support assumptions, we prove that these differences identify individual factors in a conditional Gaussian model, up to permutation, scaling, and translation. We then develop a learner that uses the recorded strengths to estimate these patterns jointly across instances, enabling learning from few views per instance. Synthetic experiments demonstrate near-perfect factor recovery through nonlinear mixing. Controlled dSprites experiments show improved coordinate alignment using augmentations that target known factors, even without telling the learner which augmentation corresponds to which representation coordinate. These results suggest a route to component identification in self-supervised learning: augmentation design can supply identifying information beyond shared content alone.
\end{abstract}

%% file: sections/01_introduction.tex
\section{Introduction}
\label{sec:intro}

Suppose an encoder learns a sprite's horizontal and vertical position. Another learns their sum and difference. Both retain the same information, but the first represents each position through an individual coordinate. Changing either position in the second representation requires changing both coordinates. Predicting position from the whole representation cannot resolve this distinction. We ask what information an augmentation-based self-supervised learning (SSL) procedure can use to select individual content coordinates.

Shared-content SSL theory gives conditions for identifying the content preserved across augmented views, up to an invertible transformation of that block \citep{vonkugelgen2021,daunhawer2023}. In a model where every view preserves the same content block, the sharing pattern does not distinguish one content coordinate from another. Structural identification criteria can supply further distinctions through the variables a factor affects or responds to \citep{lachapelle2026sparsity,zheng2022structural}. Our baseline SSL assumptions need not supply any such distinction: we construct admissible mixtures that preserve the observations.

Component identification requires observations or model restrictions that distinguish otherwise equivalent representations. We ask how an augmentation procedure can supply those distinctions. We allow target factors to fluctuate around a stable mean for each source instance. We assign each instance a fixed vector of augmentation strengths and record it. These strengths control the components' conditional variances through component-specific response functions. We call the response slopes \emph{variance fingerprints}. Distinct fingerprints can distinguish factors even when those factors respond to the same strength variable. Which factors are separated is then decided by the augmentation protocol rather than discovered afterward in the trained encoder.

The conditional Gaussian setting makes the source of identification explicit: components are distinguished by how their variances respond to the supplied strengths. The same conditional model motivates the learner's quadratic instance-ranking score. Our dSprites experiments test this construction with either a prescribed response matrix or a fully learned matrix, while retaining factor-informed transformations and recorded strengths.

We make three contributions. First, we exhibit a conditional Gaussian construction in which constant content covariance permits observationally equivalent coordinate mixtures, so the baseline SSL assumptions leave component identification unresolved. Second, we characterize an augmentation response family that identifies components within a Gaussian comparison class. Distinct variance responses rule out mixing across components, while repeated responses leave mixing within their shared block. Third, we test how a learner uses this information in synthetic and controlled image experiments. The image studies examine matching metadata, variance regularization, learned response coefficients, and new instances. Their supplied information is factor-informed transformations and the recorded strengths, which we disclose alongside the method.

%% file: sections/02_baseline.tex
\section{What the baseline SSL assumptions leave unresolved}
\label{sec:obstruction}
We consider the baseline SSL setting in which learning has access to instance-grouped views but no supplied content-specific auxiliary variables or action information beyond that grouping. In this setting, even conditionally independent content coordinates can admit observationally equivalent mixtures. The construction below makes this ambiguity explicit and motivates the augmentation design of Section~\ref{sec:fingerprints}.

\runin{The baseline SSL setting.}
We assume that training data consist of groups of augmented observations, indexed by instance $i$ and view $r$. Each view draws content from an instance-conditioned law $P_i$ and style from a separate law conditioned on any style-only auxiliary $\tau_{ir}$:
\begin{equation}
 x_{ir}=f(\alpha_{ir},\beta_{ir}),\qquad
 P(\alpha,\beta\mid i,\tau)=P_i(\alpha)\,P_{\mathrm{sty}}(\beta\mid \tau).
 \label{eq:shared}
\end{equation}
Here $f$ is the decoder. Views are independent conditional on their instance and style auxiliaries. The observed quantities are the views, their instance grouping, and any supplied style-only auxiliary; the latent content and style are unobserved. Many SSL approaches assume $P_i=\delta_{c_i}$, so every view has $\alpha_{ir}=c_i$ \citep{vonkugelgen2021}. Our approach requires the more general case. We treat the conditional first moment $m(i)=\E[\alpha_{ir}\mid i]$ as the stable instance-specific content and retain variation around this content as information for component identification.

\input{visual_abstract}

\runin{Worked example: position jitter.}
Hold appearance fixed. For source position $m(i)$ and positive jitter variances $\sigma^2_{i1},\sigma^2_{i2}$, let each view's position satisfy
\begin{equation*}
 \alpha_{ir}=m(i)+\operatorname{diag}(\sigma_{i1},\sigma_{i2})\,\varepsilon_{ir},
 \qquad\varepsilon_{ir}\sim\mathcal N(0,I_2),
\end{equation*}
with independent draws across views. Consider the shared sum/difference transformation
\begin{equation*}
 Q=\frac1{\sqrt2}\begin{pmatrix}1&1\\1&-1\end{pmatrix},\qquad
 Q\operatorname{diag}(\sigma^2_{i1},\sigma^2_{i2})Q^\top
 =\frac12\begin{pmatrix}\sigma^2_{i1}+\sigma^2_{i2}&\sigma^2_{i1}-\sigma^2_{i2}\\\sigma^2_{i1}-\sigma^2_{i2}&\sigma^2_{i1}+\sigma^2_{i2}\end{pmatrix}.
\end{equation*}
If $\sigma^2_{i1}=\sigma^2_{i2}$, this mixture preserves conditional independence. Replacing the decoder by $\tilde f(z)=f(Q^\top z)$ preserves every observation, so the position coordinates are not identified. More generally, if $\sigma^2_{i1}/\sigma^2_{i2}$ is constant across instances, a shared diagonal rescaling followed by mixing preserves independence (Figure~\ref{fig:strength-abstract}a--b). To preview the design in Section~\ref{sec:fingerprints}, assign one recorded scalar strength $s_i$ to each instance and set $\sigma^2_{i1}=e^{s_i}$ and $\sigma^2_{i2}=e^{2s_i}$. The variance ratio changes across instances, and the displayed off-diagonal entries are nonzero whenever $s_i\ne0$. Thus the same mixture violates conditional independence (Figure~\ref{fig:strength-abstract}c). The log-variances have distinct response slopes, $1$ and $2$, despite sharing one strength; Section~\ref{sec:fingerprints} formalizes these responses as variance fingerprints.

Equation~\ref{eq:shared} is unchanged by invertible reparameterizations of content when its law and decoder are transformed together. The question is which reparameterizations remain admissible under the model assumptions.

\begin{definition}[Admissible reparameterization]
\label{def:admissible}
Relative to a specified model class, an invertible content map $\tilde\alpha=h(\alpha)$ acts on a model by transforming its conditional content law and replacing the decoder with
\[
 \tilde f(\tilde\alpha,\beta)=f(h^{-1}(\tilde\alpha),\beta).
\]
The map is shared across all instances and views; style coordinates and the auxiliary law are held fixed. This action preserves every generated view and hence the grouped-view observation law. The action is \emph{admissible} at a model if the model it produces also satisfies the specified assumptions.
\end{definition}
We measure the remaining coordinate ambiguity by distinguishing block identification \citep{vonkugelgen2021} from component-wise identification. For an injective decoder, write $\alpha_\theta(x)$ for the content coordinates obtained by inverting model $\theta$'s decoder on its observational support, and compare models of the same content dimension $n$.

\begin{definition}[Block and component identification]
\label{def:comp}
Relative to a specified model class, content is \emph{block-identified} at the true model $\theta^\star$ if every model $\theta$ in the class with the same observation law satisfies $\alpha_\theta(x)=h(\alpha_{\theta^\star}(x))$ almost surely for an invertible map $h$. It is \emph{component-identified} if every such model satisfies $\alpha_{\theta,j}(x)=h_j\bigl(\alpha_{\theta^\star,\pi(j)}(x)\bigr)$ almost surely for a permutation $\pi$ and invertible scalar maps $h_j$.
\end{definition}
Population prediction from the whole vector is unchanged by invertible mixing, so block identification does not pin down coordinates (Appendix~\ref{app:metrics}).

Invertible content maps act on models as a group, but admissibility is a property at a model, so the admissible maps need not form a subgroup. The Gaussian construction below exhibits a subgroup whose entire orbit remains admissible.

\runin{Conditional independence still permits mixing.}
This ambiguity can persist even when the content coordinates are required to be conditionally independent. Specialize Equation~\ref{eq:shared} to Gaussian instance-conditioned content with a smooth injective decoder and constant positive diagonal covariance. Rotations in whitened coordinates preserve this model class, giving the Gaussian ambiguity of \citet[Proposition~1 and Supplement~B.5]{khemakhem2020ivae}. The style conditionals are held fixed, and the observation law contains the grouped views, instance indices, and any style-only auxiliary, with no recorded content strengths or realized content shifts.

\begin{lemma}[Constant conditional covariance permits mixing]
\label{lem:homo}
Suppose $\alpha\mid i\sim\mathcal N(m(i),\Sigma)$, with $\Sigma$ positive diagonal and constant across instances. The matrices $M_Q=\Sigma^{1/2}Q\Sigma^{-1/2}$, $Q\in O(n)$, form a covariance-preserving subgroup $H_\Sigma$. Since $M_Q\Sigma M_Q^\top=\Sigma$, each transformed model has content means $M_Qm(i)$, the same conditional covariance, and decoder $f_Q(\alpha,\beta)=f(M_Q^{-1}\alpha,\beta)$. Its grouped-view law is unchanged. For $n\ge2$, this orbit contains non-component-wise alternatives.
\end{lemma}
The transformed covariance stays diagonal, so each alternative is again in the class. The orbit consists of observationally equivalent models obtained by these coordinate changes. The members that mix content coordinates show that the stated assumptions do not identify individual factors in the sense of Definition~\ref{def:comp}. Appendix~\ref{app:models} gives the proof.

\runin{Non-identifiability under shared structural constraints.}
The same ambiguity survives additional structural restrictions: Appendix~\ref{app:models} constructs mixtures that preserve identical action-parent relationships \citep{lachapelle2026sparsity} or common decoder support \citep{zheng2022structural,zheng2023sparsity}. In SSL, grouping augmented views by instance therefore does not guarantee component identification even with these additional constraints. In these constructions, content has constant conditional covariance across instances, leaving only the means to vary. Section~\ref{sec:fingerprints} instead designs augmentations whose conditional variance ratios vary across instances, so that mixing content coordinates need not preserve conditional independence.

%% file: visual_abstract.tex
\begin{figure}[t]
\centering
\includegraphics[width=.92\linewidth]{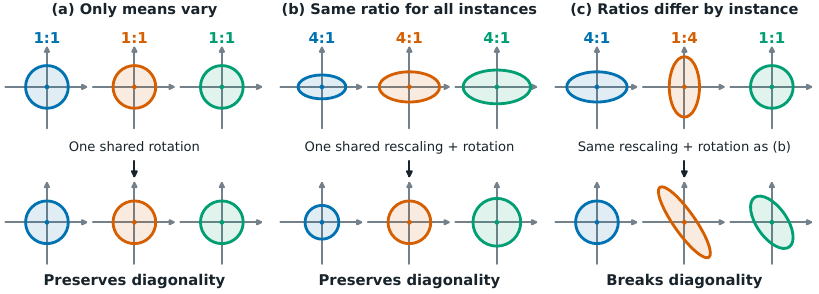}
\caption{\textbf{Variance ratios distinguish coordinates.}
Each color denotes an instance; contours show position distributions centered at the instance means.
Top: original coordinates and horizontal-to-vertical variance ratios.
Bottom: transformed coordinates under one shared linear map per panel.
\textbf{(a)} Equal variances permit orthogonal mixing.
\textbf{(b)} A common ratio permits shared rescaling and mixing.
\textbf{(c)} Under varying ratios, the map from (b) no longer preserves diagonality across instances.
Contours are schematic; nonlinear identification requires Theorem~\ref{thm:fingerprint}'s assumptions.}
\label{fig:strength-abstract}
\end{figure}

%% file: sections/03_identification.tex
\section{Designed variance responses identify components}
\label{sec:fingerprints}

\runin{Stable means, varying variances.}
We treat augmented views as noisy measurements of stable instance content $m(i)$, with conditional variances describing augmentation-induced uncertainty. To distinguish content coordinates, we exploit differences in how their conditional variances respond to augmentation strengths. Specifically, we draw and record one augmentation strength vector $s_i\sim\mathcal N(0,\omega^2I_d)$ per instance, fixed throughout training. As in ordinary SSL, each view is generated by an independent augmentation draw, but here the augmentation distributions vary per instance according to the recorded strengths. The scale $\omega$ controls strength variation across instances. We model each component's response as
\begin{equation}
 \log\Var(\alpha_j\mid i)=v_{ij}=b_{0j}+b_j^\top s_i.
 \label{eq:response}
\end{equation}
The slope $b_j$ is component $j$'s \emph{variance fingerprint}; $B$ collects these vectors as rows. A learner can use the recorded strengths to estimate these fingerprints and distinguish components through their different responses.

\runin{Component identification.}
To establish component identification, we specialize Equation~\ref{eq:shared} to the Gaussian response model of Equation~\ref{eq:response}. The observed data are grouped views $(i,s_i,\{(x_{ir},\tau_{ir})\}_{r=1}^{R})$. The recorded strength $s_i$ is fixed for each instance; the latent content $\alpha_{ir}$ remains unobserved. The factorization and conditional independence of views follow Section~\ref{sec:obstruction}.

We now establish when distinct variance fingerprints identify individual content components. This requires specifying both the competing representations and the population variation available to distinguish them. The following assumptions make these requirements precise.
\begin{assumption}[Gaussian comparison class, decoder, and support]
\label{ass:ivae}
\leavevmode\par
\begin{enumerate}[label=(\roman*),leftmargin=*,nosep]
\item \textbf{Model class.} We compare models with the same content and style dimensions, the factorization in Equation~\ref{eq:shared}, and smooth injective decoders with no additional observation noise. In every model, the content coordinates are conditionally independent nondegenerate Gaussians given the instance. Competing means and diagonal covariances may depend arbitrarily on the instance; competing covariances need not satisfy Equation~\ref{eq:response}.
\item \textbf{Population support.} In the true population, $(m(i),s_i)$ has support $\mathcal M\times\mathcal S$, where $\mathcal M$ affinely spans $\R^n$ and $\mathcal S$ has nonempty interior.
\item \textbf{Style overlap and conditional variation.}
Style coordinates may be dependent and need not follow a parametric distribution. For almost every style-only auxiliary value $t$ under the distribution of $\tau$, the conditional style distributions $P_{\mathrm{sty}}(\cdot\mid\tau=t)$ have the same null sets, and the conditional support of $(m(i),s_i)$ given $\tau=t$ is $\mathcal M\times\mathcal S$.
\end{enumerate}
\end{assumption}
Assumption~\ref{ass:ivae} allows mean variation to distinguish content directions and strength variation to distinguish variance responses, while the style conditions ensure a common content correspondence across auxiliary values. Product support requires coverage of every pairing of content means and augmentation strengths, not statistical independence. Identification is within the Gaussian comparison class, although only the true variances must satisfy Equation~\ref{eq:response}.

Theorem~\ref{thm:fingerprint} characterizes the remaining ambiguity within the Gaussian comparison class: every observationally equivalent model has content coordinates related by an affine map that mixes only components with equal fingerprints. Distinct fingerprints therefore identify individual components, while repeated fingerprints identify blocks.
\keeplines{10}
\begin{theorem}[Identification from variance fingerprints]
\label{thm:fingerprint}
Let the true conditional content law be $p(\alpha\mid i)=\prod_j\mathcal N(\alpha_j;m_j(i),e^{v_{ij}})$ with $v_{ij}$ given by Equation~\ref{eq:response}. Under Assumption~\ref{ass:ivae}, the following hold:
\begin{enumerate}[label=(\roman*),leftmargin=*,nosep]
\item \textbf{Natural-parameter rank.} The content natural-parameter differences span $\R^{2n}$ exactly when the fingerprints are pairwise distinct and nonzero.
\item \textbf{Component identification.} If the fingerprints are pairwise distinct, content is component-identified in the sense of Definition~\ref{def:comp}, up to scalar affine maps. This includes a singleton zero fingerprint. The component correspondence does not mix true style into recovered content.
\item \textbf{Partial identification.} In general, the classes of equal fingerprints are identified as blocks: every observationally equivalent model satisfies $\hat\alpha=A\alpha+a$ almost surely, with $A\in\GL(n)$ block diagonal on these classes after an output permutation. The correspondence does not depend on true style. For each class $C$, the matrix $A_C\Sigma_CA_C^\top$ is positive diagonal, where $\Sigma_C=\operatorname{diag}(e^{b_{0j}})_{j\in C}$. In each repeated class, the covariance-preserving subgroup contains non-component-wise transformations preserving the assumptions and observation law.
\end{enumerate}
\end{theorem}
\emph{Proof sketch:} Part (i) follows from mean variation and independence of exponential responses. For parts (ii)--(iii), product support supplies Gaussian laws at a common strength with affinely independent means. Style cancels from their likelihood ratios, so a true log-ratio, affine in Gaussian content, equals a competing log-ratio that is quadratic; a quadratic polynomial of a nondegenerate Gaussian is Gaussian only if its quadratic part vanishes, so the correspondence is affine. Diagonal covariances then restrict mixing to equal-fingerprint classes, and whitened rotations within repeated classes give the remaining ambiguity. Appendix~\ref{app:fingerprint} gives the full proof.

\runin{Implications for augmentation design.}
A single strength can distinguish many variance responses, so $B$ need not have full row rank. A unique zero fingerprint also remains identifiable. Thus nonlinear ICA's full natural-parameter rank condition is not necessary here.

%% file: sections/04_learning.tex
\section{Learning with recorded strengths}
\label{sec:objective}

Our learner follows the memory-bank instance-discrimination framework of \citet{wu2018instance}, using Gaussian scores with strength-dependent variances and an additional within-instance variance-matching loss. The affine response in Equation~\ref{eq:response} shares variance parameters across instances, allowing the learner to estimate their responses from few views. Content means are estimated separately.

\runin{Inputs and response model.}
The learner receives instance-grouped augmented views and one recorded strength vector $s_i\in\R^d$ per instance, fixed throughout training. The response model in Equation~\ref{eq:response} has learned intercepts and a matrix $B$ that may be prescribed or learned. Prescribing $B$ supplies the response coefficients and their assignment to content coordinates. Training uses the recorded strengths, not the realized per-view augmentation shifts or semantic factor values.

\runin{Content codes and instance means.}
A memory bank stores estimated instance content means, held fixed during each gradient update, as reference centers for instance ranking. We refresh it from newly sampled views each epoch and update it by an exponential moving average after each minibatch. At each refresh, we divide codes and stored means by the coordinate-wise standard deviations of the raw instance means. These scales remain fixed between refreshes; gradients pass through neither them nor the bank. We initialize the log-variance intercepts from the untrained encoder's within-instance dispersion in these calibrated units. Appendix~\ref{app:gauge} gives the details.

\runin{Training objective.}
Squared-distance alignment penalizes within-instance dispersion (Equation~\ref{eq:identity}, Appendix~\ref{app:alignment}). To distinguish instances while retaining the designed fluctuations, we combine instance discrimination with variance matching:
\begin{equation*}
 \mathcal L=\mathcal L_{\mathrm{sep}}+\gamma\mathcal L_{\mathrm{var}},
 \qquad \mathcal L_{\mathrm{var}}=\mathcal L_{\mathrm{logvar}}+\eta\mathcal L_{\mathrm{corr}}.
\end{equation*}
Here $\mathcal L_{\mathrm{sep}}$ ranks the source instance of each view, $\mathcal L_{\mathrm{logvar}}$ matches its variance response, and $\mathcal L_{\mathrm{corr}}$ targets zero within-instance cross-covariances. The weights $\gamma$ and $\eta$ control the variance terms and the relative contribution of decorrelation.

\runin{Gaussian instance ranking.}
For calibrated content $y$, stored candidate mean $m_u$, and recorded candidate strengths $s_u$, the ranking score is
\begin{equation}
 \ell_u(y)=-\frac12\sum_j\left[(y_j-m_{uj})^2e^{-v_{uj}}+v_{uj}\right],
 \qquad v_{uj}=b_{0j}+b_j^\top s_u.
 \label{eq:score}
\end{equation}
The quadratic term measures deviations relative to the candidate's variance; the log-determinant term $\sum_j v_{uj}$ penalizes broad variances. We apply a softmax to these scores over the source instance and sampled bank negatives. Cross-entropy with the source instance as the target defines $\mathcal L_{\mathrm{sep}}$ (Algorithm~\ref{alg:step} in Appendix~\ref{app:impl}).

\runin{Variance matching.}
We split each instance's calibrated views into two equal groups, each containing at least two views. The term $\mathcal L_{\mathrm{logvar}}$ multiplies their bias-corrected log-variance residuals against $v_{ij}$. The term $\mathcal L_{\mathrm{corr}}$ multiplies corresponding off-diagonal covariances and normalizes by pooled sample variances. The split addresses the sampling bias from squaring a single estimate; the implemented loss is not generally unbiased (Appendix~\ref{app:logspace}). These terms leave instance offsets unconstrained, motivating the accompanying instance-discrimination loss (Lemma~\ref{lem:offsets}, Appendix~\ref{app:offsets}).

\runin{Connection to the identification theorem.}
The Gaussian score in Equation~\ref{eq:score} substitutes encoded content $y$ for latent content $\alpha$ in Equation~\ref{eq:shared}, with bank estimates $m_u$ for means $m(u)$. The term $\mathcal L_{\mathrm{logvar}}$ targets Equation~\ref{eq:response}, while $\mathcal L_{\mathrm{corr}}$ encourages diagonal conditional covariance.

%% file: sections/05_experiments.tex
\section{Experiments testing the design}
\label{sec:experiments}

Theorem~\ref{thm:fingerprint} describes which ambiguities variance responses resolve at the population level. We now ask whether these responses support recovery using the proposed learner and finite data. Linear and nonlinear synthetic experiments test how the response design and the number of views affect recovery, and whether the response must be supplied. Image experiments test coordinate recovery under distinct responses (Theorem~\ref{thm:fingerprint}(ii)) and factor grouping under repeated responses (Theorem~\ref{thm:fingerprint}(iii)).

\subsection{Synthetic experiments}
\label{sec:synthetic}
We generate four-dimensional Gaussian content for 2,000 instances. Each instance receives an independent mean $m(i)\sim\mathcal N(0,I_4)$ and strength vector $s_i\sim\mathcal N(0,\omega^2I_d)$, both fixed across views. Each view samples
\[
\alpha_{ir}\mid i\sim\mathcal N\!\left(m(i),\operatorname{diag}\{\exp(Bs_i)\}\right).
\]
The rows of $B$ specify the response fingerprints, and $\omega$ controls strength variation across instances. At $\omega=0$, every instance has identity content covariance, violating Assumption~\ref{ass:ivae}(ii). The linear studies observe $x_{ir}=A\alpha_{ir}$ for an unknown invertible matrix $A$. The nonlinear studies add style and replace $A$ with a neural decoder. Appendix~\ref{app:synthetic} gives the construction and training details.

\runin{Distinct responses permit recovery.}
We first test the response design with a linear likelihood estimator (Figure~\ref{fig:synthetic}A). We pair a shared scalar strength with either equal or distinct slopes. Independently varying strengths act through either $B=I_4$ (one strength per component) or a fixed random $B$ (strengths affecting multiple components). Using four views, we estimate an unmixing map and separate variances for each instance and coordinate. The estimator receives only grouped views; $B$ governs data generation but is neither supplied nor fitted. Amari distance measures recovery, with zero indicating recovery up to scaling and permutation. Equal slopes retain the mixing ambiguity. Independent strengths improve recovery as their variation increases. Distinct scalar slopes supply full population natural-parameter rank but give substantially worse recovery over this sweep.

\runin{Pooling improves estimation from few views.}
We next ask whether supplying the strengths makes recovery easier when each instance provides few views. Separate variance estimation fits $Mn$ parameters for $M$ instances, while the affine response in Equation~\ref{eq:response} shares $n(d+1)$ parameters across instances. Figure~\ref{fig:synthetic}B compares these estimators as the view count varies, with true $B=I_4$ and $\omega=0.5$. Both estimate a linear unmixing map; the pooled estimator also learns the response to the recorded strengths. Pooling gives low recovery error with two views, where separate estimates perform poorly. At four views, Amari error falls from \amortPlugHalf{} to \amortDHalf{}. Appendix~\ref{app:amort} defines the estimators and tests a control in which variances vary independently of the recorded strengths. Pooling fails in that control, showing that its benefit depends on how well the strengths explain the variance variation.

\runin{Recovery with a nonlinear decoder.}
We next test whether the Gaussian-score learner can recover content through nonlinear mixing with style. We generate content with $B=I_4$, add two independent Gaussian style coordinates, and observe $x_{ir}=f(\alpha_{ir},\beta_{ir})$ through a fixed invertible three-layer network. We measure coordinate alignment by Hungarian-matched mean absolute correlation (MCC) between encoded views and their latent realizations. When the learner receives the true $B$ with eight views and three seeds, mean view MCC rises from 0.729 at $\omega=0$ to 0.999 at $\omega=0.5$ with the variance term (Table~\ref{tab:propc}), and removing the term ($\gamma=0$) also gives near-perfect alignment at $\omega=0.5$: the score itself reads the supplied responses.

We then withhold $B$, as in the dSprites study below: the learner fits a dense $4\times4$ matrix from zero, with no support pattern, so it is not told which strength drives which coordinate. With four views, no variance term and five seeds, mean view MCC is \synthDenseMccOff{} at $\omega=0$ and \synthDenseMccHalf{} at $\omega=0.5$, and no seed falls below \synthDenseMccOnLo{} at any $\omega>0$. Three seeds given the true $B$ in this same configuration reach \synthOracleMccHalf{} at $\omega=0.5$, so withholding $B$ does not change alignment; these runs use fewer views and stop early, so that reference, not Table~\ref{tab:propc}, is their comparison. The fitted $B$ recovers the assignment: a one-to-one matching collects at least \synthDenseShareLo{} of its absolute mass at every $\omega>0$, the matched slopes lie between \synthDenseSlopeLo{} and \synthDenseSlopeHi{} against the true value 1, and the inferred assignment agrees with the coordinates' latent alignment in all \synthDenseAgree{} of \synthDenseAgreeN{} seeds (Table~\ref{tab:synthdense}).

\begin{figure}[t]
\centering
\includegraphics[width=.63\linewidth]{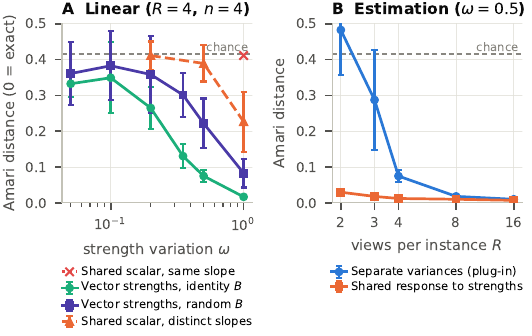}
\caption{\textbf{Linear synthetic recovery.} Amari distance (zero is exact recovery) against (A) strength variation for the four response constructions and (B) views per instance for separate and pooled variance estimation. Dashed lines: mean distance of a Haar-random rotation. Bars: one standard deviation across seeds.}
\label{fig:synthetic}
\end{figure}

\subsection{Image experiments}
\label{sec:images}
\runin{Data and interventions.}
dSprites contains binary images of 2D shapes with known generative factors \citep{matthey2017dsprites}. Translation and log-zoom target horizontal position, vertical position, and scale separately, giving three scalar factors for testing coordinate recovery. Shape and orientation vary across instances but are not targeted by these augmentations. Shapes3D contains rendered scenes with controlled object properties and hues for object, floor, and wall \citep{burgess2018shapes3d}. Region-wise hue rotations target these three hues. Because hue is periodic, we use two coordinates per hue and test whether each pair captures a hue group.

We draw $s_i\sim\mathcal N(0,\omega^2I_d)$ once per instance and hold it fixed across views. Each view then samples a shift $a_{irk}=\sigma_{0k}e^{s_{ik}}\epsilon_{irk}$ for each family $k$, with independent standard Gaussian draws $\epsilon_{irk}$. Here $\sigma_{0k}>0$ is the shift standard deviation at zero strength. Setting $\omega=0$ removes variation in strength across instances but retains random shifts within each instance. An ideal additive shift gives a slope-2 log-variance response; zoom is additive in log-scale. Appendix~\ref{app:amounts} derives this response and explains the limits of the Gaussian content model for image augmentations.

\runin{Training and evaluation.}
All image studies train a four-layer convolutional encoder on 20,000 instances, using eight views per instance and 400 epochs. For evaluation, we average encoder outputs over eight views sampled once and held fixed. These averages come directly from the images, so new instances need no bank entry.

On dSprites, MCC measures alignment with individual factors, while five-neighbor factor $R^2$ measures prediction from the whole representation. We evaluate both on 2,000 training-pool instances, splitting the pool in half to fit and test the predictor. To assess transfer beyond the training pool, we also evaluate \holdoutN{} instances excluded from encoder training, with strengths drawn independently from the same distribution. We compare alignment against random rotations and invertible mixtures of the learned representation. Appendices~\ref{app:impl} and~\ref{app:images} give the protocols and reference constructions.

\begin{figure}[t]
\centering
\includegraphics[width=\linewidth]{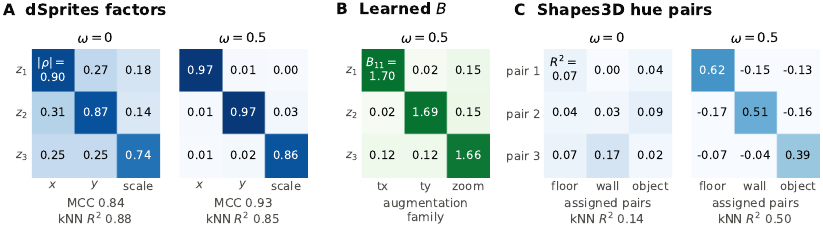}
\caption{\textbf{Image results.} (A) dSprites with a learned dense $B$: absolute correlation $|\rho|$ between learned coordinates and targeted factors, with coordinates matched to factors in each run and entries averaged over five seeds. (B) Fitted $B$ at $\omega=0.5$; rows follow (A), and columns are the families that shift $x$, $y$ and scale (ideal slope 2). (C) Shapes3D: kNN $R^2$ for predicting each hue from each prescribed coordinate pair, averaged over six seeds; negative entries are shaded as zero.}
\label{fig:images}
\end{figure}

\runin{Strength variation selects the coordinates.}
On dSprites, we test whether the learner can recover the three target factors while learning their variance responses. We fit all entries of a $3\times3$ matrix $B$ from zero, without specifying which augmentation family affects which coordinate. Table~\ref{tab:learnedresponse} reports recovery as strength variation increases. Figure~\ref{fig:images}A compares coordinate alignment with and without strength variation. Without strength variation, the coordinates mix the factors, with off-diagonal correlations up to 0.31, and mean factor MCC is 0.837. At $\omega=0.5$, the off-diagonal correlations fall to at most 0.03 and MCC rises to \trainDspMccHalf{}. Block prediction changes little (kNN $R^2$ 0.88 and 0.85), so the representation retains the factor information in both conditions; strength variation changes which coordinates carry it. The result transfers to instances excluded from encoder training (Table~\ref{tab:learnedresponse}), and all \holdoutDspIdentN{} runs at $\omega=0.2$ and $0.5$ exceed their mixing references on both instance sets.

The benefit does not extend to the largest tested strength variation. At $\omega=1$, both alignment and block prediction deteriorate sharply (Table~\ref{tab:learnedresponse}). Increasing $\omega$ also increases the expected squared shift, $\sigma_{0k}^2e^{2\omega^2}$, to about $7.4$ times its value at $\omega=0$. This comparison therefore cannot separate the effect of greater strength variation from that of larger shifts.

\runin{The learner finds the response assignment.}
Figure~\ref{fig:images}B shows the fitted $B$ at $\omega=0.5$. Each coordinate responds mainly to the family that shifts the factor it aligns with. A one-to-one matching accounts for a fraction \denseHalfShare{} of the matrix's absolute mass, and the inferred assignment agrees with factor alignment in all \denseHalfAgree{} seeds (Appendix~\ref{app:learned-response}). The matched slopes, near 1.7, stay below the ideal value of 2. Prescribing $B=2I_3$ gives a similar MCC of \trainGaussDspMcc{}, so a supplied assignment is not needed for alignment here. In that prescribed setting, permuting strengths across instances or removing the variance term lowers both alignment and block prediction (Appendix~\ref{app:dsprites-primary}, Figure~\ref{fig:dsprites-controls}).

\runin{Repeated responses support hue groups.}
Finally, we test whether repeated responses encourage grouping on Shapes3D. We prescribe two coordinates per hue, with the same fingerprint within each pair. Under the Gaussian model, this pattern identifies three two-dimensional blocks (Theorem~\ref{thm:fingerprint}(iii)). We assess the corresponding grouping behavior by predicting each hue's circular coordinates $(\cos2\pi h,\sin2\pi h)$ from each learned pair, using five-neighbor regression. Hungarian assignment matches pairs to hues. A mixing null tests whether arbitrary subspaces of the same representation predict the hues as well (Appendix~\ref{app:hue-criteria}).

With strength variation, all \sThreeSigOn{} of \sThreeSeedsN{} seeds exceed their null at $\hat p<0.01$. The mean assigned score is \sThreeAssignedMean{}, compared with \sThreeOffMean{} for off-assignment predictions. Without strength variation, \sThreeSigOff{} of \sThreeSeedsN{} seeds exceed their hue-group mixing reference, and the assigned score is \sdOmZeroAssigned{}. Averaged over seeds, each pair predicts its own hue best under strength variation (Figure~\ref{fig:images}C). Table~\ref{tab:s3dscores} in Appendix~\ref{app:shapes3d} reports the results for each seed. These results support selective, imperfect hue grouping under the designed responses. They do not test the theorem's full Gaussian setting, because a smooth two-coordinate representation of continuous hue lies on a curve. The aggregate Shapes3D scores in Table~\ref{tab:holdout} additionally include scale and orientation and do not measure this grouping.

\begin{table}[t]
\centering\small
\caption{\textbf{Learned dense responses.} Mean $\pm$ standard deviation over five optimization seeds with a shared strength design.}
\label{tab:learnedresponse}
\begin{tabular}{@{}crrrr@{}}\toprule
 & \multicolumn{2}{c}{Factor MCC} & \multicolumn{2}{c}{kNN $R^2$} \\
$\omega$ & training pool & held-out & training pool & held-out \\\midrule
0 & 0.837 $\pm$ 0.065 & 0.838 $\pm$ 0.064 & 0.881 $\pm$ 0.002 & 0.877 $\pm$ 0.008 \\
0.2 & 0.928 $\pm$ 0.023 & 0.927 $\pm$ 0.023 & 0.877 $\pm$ 0.001 & 0.872 $\pm$ 0.007 \\
0.5 & 0.932 $\pm$ 0.001 & 0.931 $\pm$ 0.001 & 0.848 $\pm$ 0.003 & 0.847 $\pm$ 0.007 \\
1 & 0.579 $\pm$ 0.247 & 0.581 $\pm$ 0.242 & 0.394 $\pm$ 0.275 & 0.403 $\pm$ 0.258 \\
\bottomrule\end{tabular}
\end{table}

%% file: sections/06_discussion.tex
\section{Related work}
\label{sec:related-work}

\runin{Sharing patterns and structure.}
Identification results separate factors only as finely as the information distinguishing them. Shared-content SSL identifies preserved content up to an invertible block transformation \citep{vonkugelgen2021,daunhawer2023}. Since a factor shared by all views stays mixed with the others, augmentation-based and multi-domain results identify individual style variables while content remains a block \citep{eastwood2023,kong2022partial,xie2023multidomain}. Structural criteria based on the variables each factor is tied to \citep{lachapelle2026sparsity,zheng2022structural} leave factors with identical relationships mixed, as in our baseline constructions.

\runin{Designed conditional variability.}
Auxiliary-variable nonlinear ICA separates factors through changes in their conditional distributions \citep{hyvarinen2019auxiliary,khemakhem2020ivae}, variation it takes from the data source. The instance index is such a variable, but constant conditional covariance still permits mixing (Lemma~\ref{lem:homo}), and SSL alignment pushes within-instance dispersion toward zero (Appendix~\ref{app:alignment}). We instead design the variation, keeping instance means stable while recorded strengths vary conditional variances. Nonstationary linear ICA fails for proportional variance profiles \citep{pham2001blind}, the linear analogue of repeated fingerprints.

\runin{Recorded augmentation information.}
Equivariant and augmentation-aware methods use per-view transformation parameters as prediction targets or inputs \citep{garrido2023sie,devillers2023equimod,lee2021augself}, without an identification guarantee. Our strengths are fixed per instance and govern its augmentation distribution, and the learner uses them to model conditional variances without observing realized shifts. Appendix~\ref{app:related-work} also contrasts our assumptions with contrastive identification results \citep{zimmermann2021,rusak2025infonce}.

\section{Discussion}
\label{sec:conclusion}
Shared content alone can leave individual factors indistinguishable from mixtures. Our results show how augmentation design can supply the information needed to distinguish them. In the Gaussian comparison class, varying recorded strengths across instances rules out mixing between factors with distinct variance responses, while repeated responses retain ambiguity within blocks (Theorem~\ref{thm:fingerprint}). The learner uses those same strengths to pool variance estimation across instances, making the designed variation useful even with few views. Synthetic experiments demonstrate recovery through nonlinear mixing. On dSprites, moderate strength variation aligns individual coordinates more closely with the factors, while prediction from the full representation changes little. The learner discovers the coordinate assignment. On Shapes3D, coordinate pairs with prescribed, repeated fingerprints each predict their assigned hue best on average. Together, the theory and controlled experiments motivate treating component identification in SSL as an experimental design problem: choosing augmentation distributions and recording their strengths can supply information that shared content alone leaves unresolved.

\runin{Limitations and future work.}
The guarantee concerns a Gaussian comparison class, not recovery by the trained learner; the image experiments do not establish its assumptions (Appendix~\ref{app:limits}). Their augmentations target known factors, even when the coordinate assignment is learned. Comparisons are limited to ablations; we do not test for augmentation shortcuts, quantify the cost of eight views, or use natural images. Future work could study how to choose strengths for reliable recovery and what identifying information richer SSL augmentation pipelines supply. Larger, more complex datasets would help assess applicability beyond transformations designed to target individual factors.

%% file: sections/07_ai_use.tex
\subsection*{AI use statement}
We used large language model (LLM) assistance in this work. We disclose it below by the categories of the ICLR AI policy, followed by how the assisted work was checked.

\paragraph{Tasks with required disclosure.} LLM assistance contributed to the following.
\begin{itemize}[leftmargin=*,nosep]
\item \emph{Formulating mathematical claims:} the observationally equivalent construction of Section~\ref{sec:obstruction}.
\item \emph{Proofs:} \theauthors{} chose the assumptions and the claims of Lemma~\ref{lem:homo} and Theorem~\ref{thm:fingerprint}. The proof techniques were heavily informed by LLM assistance, which also drafted the arguments of Appendix~\ref{app:proofs} and the objective analysis of Appendix~\ref{app:estimation}; \theauthors{} then checked every proof by hand.
\item \emph{Proposing hypotheses and designing methodology and experiments:} \theauthors{} identified that conditioning on the instance alone leaves content unidentified. The instance-specific augmentation strength scheme emerged from brainstorming with an LLM about augmentation-based remedies, and we do not attempt to separate the LLM's contribution to that idea from our own. The affine log-variance response family of Equation~\ref{eq:response} is \theauthorsp{} own; the learner's regularization of variances in log space was an LLM suggestion, made after early experiments that regularized covariances directly proved scale-sensitive. The controls and ablations were chosen by \theauthors.
\item \emph{Implementing methods:} the training, evaluation and manuscript workflow, including the code for the coordinate-unit symmetry of the ranking score in Appendix~\ref{app:gauge}.
\end{itemize}
We did not use generative AI to generate synthetic data sets, to develop theoretical models or conceptual frameworks, to clean or reformat data sets, to support qualitative or thematic analysis, to interpret results, or for translation.

\paragraph{Tasks with recommended disclosure.} We also used LLM assistance for brainstorming, identifying gaps in the literature, literature search, drafting and editing prose, generating figures and tables, and suggestions on the title and structure. In addition, we used LLM-generated adversarial reviews to search for weaknesses. Their findings were treated as claims to verify rather than as conclusions: each was checked against the code, the stored run records or the proofs before any statement in the paper changed, and several were rejected on inspection.

\paragraph{Review and responsibility.} \Theauthors{} checked every theorem and proof by hand; no proof assistant or other formal verification was used. LLM-written code is covered by the test suite and was sanity-checked manually by \theauthors. Every reported table, figure and numerical macro is rendered from stored run records by the declared workflow, and the interpretation of the results is \theauthorsp{} own. \Theauthors{} takes{} responsibility for the final content of this work, including the claims, the correctness of the proofs, the experimental design and all released artifacts.

%% file: sections/08_reproducibility.tex
\subsection*{Reproducibility statement}
All experiments run as one Snakemake workflow that trains every run, evaluates it, and renders every table, figure and numerical macro of this manuscript from the resulting run records. Each run is a function of the committed configuration file and its own seeds and writes one record carrying that configuration, the seeds, the code commit and the package versions; trained encoders save their final weights beside the record. Assumption~\ref{ass:ivae} states the conditions of the identification result, Section~\ref{sec:obstruction} those of the observationally equivalent construction, and Appendix~\ref{app:proofs} gives the proofs. Appendix~\ref{app:synthetic} specifies the synthetic constructions and their training, and Appendix~\ref{app:impl} the image data, encoder, training budget and evaluation, including the evaluation on held-out instances reported in Tables~\ref{tab:learnedresponse} and~\ref{tab:holdout}. The synthetic data are generated by the workflow; the image benchmarks are the public dSprites and Shapes3D releases, subsampled as stated in Appendix~\ref{app:impl}.
The code, the workflow with its configuration and environments, the run records behind the reported tables and figures, and a script that regenerates them from those records will be released publicly.

%% file: appendices/00_related_work.tex
\section{Extended related work}
\label{app:related-work}

\runin{What distinguishes one factor from another.}
Independence alone does not identify latent factors through a nonlinear decoder \citep{hyvarinen1999nonlinear}, and \citet{locatello2019challenging} extend this impossibility to unsupervised disentanglement. Positive results add information that treats factors differently, and each separates factors only as finely as that information does. The literature draws this information from the factors that views share, the variables each factor is structurally tied to, and the way factor distributions change with an auxiliary variable. In the baseline SSL setting, none of these need distinguish one content factor from another (Section~\ref{sec:obstruction}). Our design adds a variance response to recorded strengths, and Theorem~\ref{thm:fingerprint} follows the same rule, since distinct fingerprints identify components and repeated fingerprints leave blocks.

\runin{Sharing patterns.}
Shared-content SSL identifies the content preserved across views up to an invertible transformation of the whole block \citep{vonkugelgen2021,daunhawer2023}. Weakly supervised pairs and partially overlapping views separate factors as finely as their sharing patterns distinguish them \citep{locatello2020weakly,ahuja2022sparse,yao2024}. A factor shared by every view stays grouped with every other such factor, so content remains a block in augmentation and multi-domain results, even where individual style variables are identified \citep{eastwood2023,timilsina2026csdi}. Our views share every target factor's instance mean, and Lemma~\ref{lem:homo} shows that mixtures then remain admissible even under conditional independence. Whole-vector regression scores, a common content metric, cannot detect these mixtures (Appendix~\ref{app:metrics}).

\runin{Structural relationships.}
Mechanism sparsity, decoder sparsity and interventions separate factors by the actions, observed variables or intervention targets tied to each \citep{lachapelle2022mechanism,lachapelle2026sparsity,zheng2022structural,zheng2023sparsity,ahuja2023interventional,vonkugelgen2023nonparametric}. They leave factors with identical relationships mixed \citep{lachapelle2026sparsity}. In our baseline constructions, content factors share the instance as their only parent and can affect the same observed variables, and Gaussian mixtures preserving these relationships remain admissible (Appendix~\ref{app:models}). A single recorded strength may act on several factors, which distinct response slopes rather than distinct targets then separate.

\runin{Distributional variation.}
Auxiliary-variable nonlinear ICA distinguishes factors by how their conditional distributions change with a time segment, a domain or a latent state \citep{hyvarinen2016tcl,hyvarinen2019auxiliary,halva2020hmm,khemakhem2020ivae}. Constant Gaussian covariances violate its variability condition and permit rotations \citep[Proposition~1 and Supplement~B.5]{khemakhem2020ivae}, which Lemma~\ref{lem:homo} extends from equal to unequal variances. Multi-domain content-style models inherit this division, identifying changing factors individually and invariant content as a block \citep{kong2022partial,xie2023multidomain}. In this line the identifying variation is a property of the data source. We condition instead on the instance and create variance variation through augmentation design, keeping instance means stable. The closest precedent is linear separation of nonstationary Gaussian sources, whose separation condition excludes proportional variance profiles \citep{pham2001blind}. Repeated fingerprints are their nonlinear counterpart.

\runin{Identification by contrastive objectives.}
With a uniform latent marginal and a von Mises--Fisher positive-pair conditional, InfoNCE recovers latents up to an orthogonal transformation \citep{zimmermann2021}, and anisotropic conditionals separate groups of factors with unequal concentration \citep{rusak2025infonce}. In these analyses each factor's conditional spread is the same for every anchor, and the restricted marginal is what excludes non-orthogonal mixtures. Our Gaussian comparison class leaves instance means unrestricted, so constant spreads admit the non-orthogonal mixtures of Lemma~\ref{lem:homo}, which variance ratios that change across instances then rule out.

\runin{Recorded strengths in learning.}
Equivariant and augmentation-aware SSL methods predict augmentation parameters from representations, or representation changes from those parameters \citep{dangovski2022essl,lee2021augself,garrido2023sie,devillers2023equimod}. Their parameters describe individual views and are drawn afresh for each. Our strengths are fixed per instance and govern its augmentation distribution, and the learner uses them to model conditional variances without observing the realized shift in any view. They also make few views workable. A per-instance variance estimated from two views is an incidental parameter too noisy to reveal its response (\citealp{neyman1948}; Appendix~\ref{app:estimation}), and response coefficients shared across instances pool this information (Section~\ref{sec:objective}). Our image augmentations still target known factors, so the difference lies in what the learner receives, not in what designing the augmentations requires.

%% file: appendices/01_proofs.tex
\section{Proofs}
\label{app:proofs}

\subsection{Proof of Lemma~\ref{lem:homo} and two structural consequences}
\label{app:models}

Throughout, specialize Equation~\ref{eq:shared} to Gaussian content and a smooth injective decoder $f$ mapping $z=(\alpha,\beta)\in\R^n\times\R^{n_\beta}$ to $x$; the style conditionals and the auxiliary law are held fixed.

\begin{proof}[Proof of Lemma~\ref{lem:homo}]
Let $p(\alpha\mid i)=\mathcal N(m(i),\Sigma)$ with $\Sigma$ positive diagonal. For $Q\in O(n)$ set $M_Q=\Sigma^{1/2}Q\Sigma^{-1/2}$, $f_Q=f\circ(M_Q\oplus\mathrm{id}_\beta)^{-1}$ and $p_Q(\alpha\mid i)=\mathcal N(M_Qm(i),\Sigma)$. Since $M_QM_{Q'}=M_{QQ'}$ and $M_Q^{-1}=M_{Q^\top}$, these matrices form a group $H_\Sigma$. The identity $M_Q\Sigma M_Q^\top=\Sigma$ shows that $M_Q\alpha$ has conditional covariance $\Sigma$ under the truth, so $p_Q$ is again a product of nondegenerate scalar Gaussian conditionals, and $f_Q$ is smooth and injective because $f$ is. The action of Definition~\ref{def:admissible} preserves the grouped-view law, so every member of the orbit is an admissible alternative in the class.

It remains to show that these alternatives are not component-wise when $Q$ is not a signed permutation. A linear correspondence $\tilde\alpha=G\alpha$ with $G\in\GL(n)$ is component-wise in the sense of Definition~\ref{def:comp} only if $G$ is a generalized permutation, a permutation matrix times an invertible diagonal matrix. Suppose instead that row $j$ of $G$ has two nonzero entries, and fix any index $l$. Given an instance, the coordinates of $\alpha$ are independent nondegenerate Gaussians, so $X=\sum_{k\ne l}G_{jk}\alpha_k$ is a nondegenerate Gaussian independent of $\alpha_l$. If $(G\alpha)_j=h(\alpha_l)$ held almost surely for a measurable $h$, then $X=h(\alpha_l)-G_{jl}\alpha_l$ would be a function of $\alpha_l$ and independent of it, hence almost surely constant, a contradiction. The constructed alternative has correspondence $\tilde\alpha=M_Q\alpha$. Since $\Sigma^{\pm1/2}$ are diagonal, $M_Q$ is a generalized permutation exactly when $Q$ is, and an orthogonal generalized permutation is a signed permutation. For $n\ge2$ a rotation by $\pi/4$ in one coordinate plane is not a signed permutation, so the orbit contains observationally equivalent models in the class whose coordinates mix those of the truth.
\end{proof}

Lemma~\ref{lem:homo} extends the Gaussian ambiguity of \citet[Proposition~1, proved in Supplement~B.5]{khemakhem2020ivae}. Their proof rotates a rotationally invariant base measure, which forces every component to have the same variance; conjugating the rotation by $\Sigma^{1/2}$ removes that restriction. Their proposition shows that the linear indeterminacy cannot be reduced to a permutation. Definition~\ref{def:comp} also allows invertible scalar maps, so the second paragraph of the proof is what turns the orbit into a failure of component identification. Constant covariance also leaves the natural parameters rank-deficient; Theorem~\ref{thm:fingerprint}(i) states when designed responses restore full rank.

\paragraph{Preserving action parents and decoder support.} In the setting of Lemma~\ref{lem:homo}, let $n\ge2$ and $M_Q=\Sigma^{1/2}Q\Sigma^{-1/2}$.

\emph{Action parents.} Suppose the means $m(i)$ affinely span $\R^n$ and instance identity is one grouped action parent of the content block, with separate style auxiliaries. Then every $M_Q$ preserves these parent relationships. This specializes the identical-parent ambiguity of \citet{lachapelle2026sparsity} while retaining conditional factorization.

\emph{Decoder support.} Call the functional support of a content-Jacobian column the set of output coordinates where its derivative is not identically zero. If all content columns share a common support $S$, then Haar-almost every $Q\in O(n)$ preserves that support in every column. Shared support is the case that the structural-sparsity conditions of \citet{zheng2022structural} and \citet{zheng2023sparsity} exclude, so this is a restricted ambiguity for support-based criteria, not a counterexample to their results.

\begin{proof}
Each row of $M_Q$ is nonzero, so its projection of an affinely spanning mean set is not constant: every transformed content coordinate still depends on the grouped instance action, and the style dependencies are unchanged. For the support claim, write $J$ for the content Jacobian; then $J_Q(\tilde\alpha,\beta)=J(M_Q^{-1}\tilde\alpha,\beta)M_Q^{-1}$, and the transformed style Jacobian is $\partial f/\partial\beta$ evaluated at the reparameterized point $(M_Q^{-1}\tilde\alpha,\beta)$. Rows outside $S$ remain zero. For each row in $S$, choose a point where that row is nonzero; vanishing of a transformed entry at the corresponding point imposes a nontrivial linear constraint on a row of $Q$, a Haar-null event, and a finite union over output rows and content columns is still null. Signed permutations form a finite set, so Haar-almost every $Q$ both preserves the support and, by the proof of Lemma~\ref{lem:homo}, gives a non-component-wise alternative.
\end{proof}

Hence any criterion that scores a model only through its action graph or its decoder support gives these alternatives the truth's score.

\subsection{Proof of Theorem~\ref{thm:fingerprint}}
\label{app:fingerprint}

\begin{proof}[Proof of Theorem~\ref{thm:fingerprint}]
Write $D(s)=\operatorname{diag}(e^{b_{0j}+b_j^\top s})$ for the true conditional content covariance at strength $s$, $w(s)$ for the vector of precisions $w_j(s)=e^{-b_{0j}-b_j^\top s}$, and $\lambda(m,s)=(m\odot w(s),\,-\tfrac12w(s))\in\R^{2n}$ for the content natural parameters, with $\odot$ the component-wise product.

\paragraph{Part (i): natural-parameter rank.} Full rank is the sufficient-variability condition of \citet[Theorem~1(iv)]{khemakhem2020ivae}; parts (ii) and (iii) do not use it. The differences $\lambda(m,s)-\lambda(m',s')$ range over $(m,s),(m',s')\in\mathcal M\times\mathcal S$. For fixed $s$, the differences $((m-m')\odot w(s),0)$ span $\R^n\times\{0\}$, because the means affinely span $\R^n$ and $w(s)$ has no zero entry. Subtracting such vectors from a difference at two strengths leaves $(0,-\tfrac12(w(s)-w(s')))$, so the span is $\R^{2n}$ if and only if the differences $w(s)-w(s')$ over $s,s'\in\mathcal S$ span $\R^n$, that is, if and only if no nonzero $c\in\R^n$ makes $c^\top w(s)$ constant on $\mathcal S$.

Suppose the $b_j$ are pairwise distinct and nonzero. Pick $u\in\R^d$ with $0,b_1^\top u,\dots,b_n^\top u$ pairwise distinct (the excluded $u$ lie in finitely many hyperplanes) and a segment $\{s_0+\rho u\}$ inside the interior of $\mathcal S$. A relation $c_0+\sum_jc_je^{-b_{0j}-b_j^\top s_0}e^{-(b_j^\top u)\rho}=0$ on an interval is a linear relation among exponentials with pairwise distinct exponents; their Wronskian is a nonzero Vandermonde determinant, so every coefficient vanishes, and $c^\top w$ is constant on $\mathcal S$ only for $c=0$. Conversely, $b_j=b_k$ makes $c^\top w\equiv0$ for $c=e^{b_{0j}}e_j-e^{b_{0k}}e_k$, and $b_j=0$ makes $w_j\equiv e^{-b_{0j}}$ constant; either way the span loses a dimension.

\paragraph{Parts (ii) and (iii): the affine block correspondence.} Let $\theta$ be a model in the comparison class with the same observation law as the truth $\theta^\star$, and write $\alpha=\alpha_{\theta^\star}(x)$ and $\hat\alpha=\alpha_\theta(x)$ for the two content coordinates of the same observation. Four steps show that $\hat\alpha=A\alpha+a$ almost surely, with $A\in\GL(n)$ block diagonal on the fingerprint classes after a row permutation and $A_C\Sigma_CA_C^\top$ positive diagonal on each class $C$. Part (ii) and the block statement of part (iii) follow; the residual ambiguity of part (iii) is proved last.

\emph{Step 1: one correspondence for every law.} Let $\Phi=\hat f^{-1}\circ f$ map the true latent vector to the competing one on the same observation. It is defined almost surely because the two models have the same observation law, and injective Borel decoders make it a bimeasurable bijection between sets of full measure \citep[Theorem~15.1]{kechris1995}; smoothness is not needed below. Write $\hat\alpha=\varphi(\alpha,\beta)$ for its content part, without assuming that $\hat\alpha$ depends on $\alpha$ alone.

We fix one style value at which every property we need holds. Write $q_t=P_{\mathrm{sty}}(\cdot\mid\tau=t)$. For $P_\tau$-almost every $t$, three properties hold: the two models have the same conditional observation law given $(i,\tau=t)$ for almost every instance under the conditional instance law at $t$; the conditional support of $(m(i),s_i)$ given $\tau=t$ is $\mathcal M\times\mathcal S$; and $q_t$ has the same null sets as $q_{t'}$ for $P_\tau$-almost every $t'$ (Assumption~\ref{ass:ivae}(iii)). Fix $t_0$ at which all three hold and write $q=q_{t_0}$; without a style auxiliary, $\tau$ is constant and $t_0$ is its value. Every true content conditional is a nondegenerate Gaussian, so the true latent law has the same null sets as $\gamma_n\otimes q$, with $\gamma_n$ the standard Gaussian, and so do the fixed-strength laws $p_k\otimes q$ of Step~2. Hence $\Phi$ is defined almost surely under each of them.

\emph{Step 2: conditional laws at a common strength.} Fix $s_0\in\mathcal S$ and affinely independent $m_0,\dots,m_n\in\mathcal M$, and write $D_0=D(s_0)$ and $p_k=\mathcal N(m_k,D_0)$. No instance need sit exactly at $(m_k,s_0)$, so we approach that point: every neighborhood of a support point has positive probability given $\tau=t_0$, so there are instances $i_r$ with $(m(i_r),s_{i_r})\to(m_k,s_0)$ at which the two models' conditional laws agree. At $t_0$ each model's style law is fixed, $q$ in the true model and $\hat q$ in the competing one, which need not coincide, and Equation~\ref{eq:shared} gives
\[
 P_{i_r}=\mathcal N\bigl(m(i_r),D(s_{i_r})\bigr)\otimes q,\qquad
 \Phi_\#P_{i_r}=\hat p_{i_r}\otimes\hat q,\qquad
 \hat p_{i_r}=\mathcal N(\nu_{i_r},V_{i_r}).
\]
Total variation contracts under pushforward and under marginalization, and a product with a fixed law preserves it. Hence, in total variation,
\[
 P_{i_r}\to P_k:=p_k\otimes q
 \;\Longrightarrow\;\hat p_{i_r}\otimes\hat q\to\Phi_\#P_k
 \;\Longrightarrow\;\hat p_{i_r}\to\hat p_k,
\]
where $\hat p_k$ is the content marginal of $\Phi_\#P_k$, and uniqueness of limits gives $\Phi_\#P_k=\hat p_k\otimes\hat q$. The limit $\hat p_k$ is Gaussian with diagonal covariance, because Gaussian laws are closed under weak limits and every $V_{i_r}$ is diagonal. It is nondegenerate because each $\hat p_{i_r}$ is equivalent to Lebesgue measure and a total-variation limit inherits absolute continuity. So $\hat p_k=\mathcal N(\nu_k,V_k)$ with $V_k$ positive diagonal.

\emph{Step 3: the correspondence is affine.} $\Phi$ is a bimeasurable bijection carrying $p_k\otimes q$ to $\hat p_k\otimes\hat q$ for every $k$, so it carries likelihood ratios too: the ratio of conditions $k$ and $0$ at $(\alpha,\beta)$ equals the competing ratio at $\Phi(\alpha,\beta)$. On each side the style factor appears in numerator and denominator and cancels, so no style density is needed. What remains is
\[
 \log\frac{p_k}{p_0}(\alpha)=\log\frac{\hat p_k}{\hat p_0}(\hat\alpha)\quad\text{almost surely},\qquad \hat\alpha=\varphi(\alpha,\beta).
\]
The left side is $(m_k-m_0)^\top D_0^{-1}\alpha+c_k$, affine in $\alpha\sim\mathcal N(m_0,D_0)$ under condition $0$, hence Gaussian. The right side is
\[
 -\tfrac12\hat\alpha^\top(V_k^{-1}-V_0^{-1})\hat\alpha+(V_k^{-1}\nu_k-V_0^{-1}\nu_0)^\top\hat\alpha+\hat c_k,
\]
a quadratic polynomial of $\hat\alpha\sim\mathcal N(\nu_0,V_0)$. A quadratic polynomial of a nondegenerate Gaussian vector is Gaussian only if its quadratic part vanishes: after whitening, a nonzero eigenvalue $\zeta$ of the quadratic coefficient makes the moment-generating function diverge at real arguments $\kappa$ with $\kappa\zeta>1/2$, whereas a Gaussian has a finite moment-generating function everywhere. Hence $V_k=V_0$ for every $k$, and the $n$ identities stack as $L\alpha=K\hat\alpha+c$ with rows $L_{k,:}=(m_k-m_0)^\top D_0^{-1}$ and $K_{k,:}=(\nu_k-\nu_0)^\top V_0^{-1}$. Affine independence makes $L$ invertible. A nonzero left-null vector $u$ of $K$ would make $u^\top L\alpha$, a nonzero linear functional of a nondegenerate Gaussian, almost surely constant; so $K$ is invertible and $\hat\alpha=A\alpha+a$ with $A=K^{-1}L\in\GL(n)$. The right-hand side involves only $\alpha$, so the competing content does not depend on true style. The identity holds $p_0\otimes q$-almost surely, and by Step~1 this law has the same null sets as the true latent law, so the identity holds almost surely in the population. Without style, $\Phi=\varphi$ is a Borel bijection between conull subsets of $\R^n$ that carries the Gaussians $p_k$ to Gaussians, and this step is the conull-set form of the affine reduction of \citet[Theorem~3.1 and Corollary~3.2]{parthasarathy2011normality}; his Remark~3.3 shows that the reduction needs bijectivity, which decoder injectivity supplies. Cancelling the style factors extends the reduction to models with style.

\emph{Step 4: no mixing across fingerprint classes.} The affine correspondence gives the competing content of instance $i$ the covariance $E(s_i)$, where $E(s)=AD(s)A^\top$. The comparison class makes $E(s_i)$ diagonal for almost every instance, hence for $s$ in a full-measure subset of $\mathcal S$; that subset is dense in $\mathcal S$ and $E$ is continuous, so $E(s)$ is diagonal on all of $\mathcal S$. Let $\Delta(s)=D(s)D_0^{-1}=\operatorname{diag}(e^{b_j^\top(s-s_0)})$ and $\hat\Delta(s)=E(s)E(s_0)^{-1}$, both diagonal. Then $E(s)=\hat\Delta(s)E(s_0)$ reads $A\Delta(s)D_0A^\top=\hat\Delta(s)AD_0A^\top$, and cancelling $D_0A^\top$ gives $A\Delta(s)=\hat\Delta(s)A$. Entrywise, $A_{kj}\ne0$ forces $e^{b_j^\top(s-s_0)}=\hat\Delta_{kk}(s)$ on $\mathcal S$. If row $k$ has nonzero entries in columns $j$ and $\ell$, then $(b_j-b_\ell)^\top(s-s_0)=0$ on a set with nonempty interior, so $b_j=b_\ell$. Every row of $A$ is therefore supported in one fingerprint class. The rows are linearly independent, so at most $|C|$ of them lie in class $C$, and $\sum_C|C|=n$ forces exactly $|C|$: after a row permutation, $A$ is block diagonal on the classes. Within a class $C$ the true covariance block is $e^{b_C^\top s}\Sigma_C$ with $\Sigma_C=\operatorname{diag}(e^{b_{0j}})_{j\in C}$, a scalar multiple of one fixed matrix, so $E$ is diagonal exactly when $A_C\Sigma_CA_C^\top$ is positive diagonal, that is, when $A_C=\hat\Sigma_C^{1/2}Q_C\Sigma_C^{-1/2}$ for a positive diagonal $\hat\Sigma_C$ and $Q_C\in O(|C|)$. Two variance profiles $e^{b_{0j}+b_j^\top s}$ and $e^{b_{0\ell}+b_\ell^\top s}$ are proportional on $\mathcal S$ exactly when $b_j=b_\ell$, so this is the uniqueness condition of exact joint diagonalization \citep[Theorem~2.3]{afsari2008sensitivity}, the nonstationary-variance condition of \citet{matsuoka1995neural} and \citet{pham2001blind}.

\paragraph{Part (ii): component identification.} If the fingerprints are pairwise distinct, every class is a singleton and every block of $A$ is a nonzero scalar, so $\hat\alpha_j=c_j\alpha_{\pi(j)}+a_j$ with $c_j\ne0$ for a permutation $\pi$: content is component-identified in the sense of Definition~\ref{def:comp}, up to scalar affine maps. A zero fingerprint enters Step~4 like any other value, so a singleton zero class is identified as well, and Step~3 shows that the correspondence does not depend on true style.

\paragraph{Part (iii): partial identification.} Steps 1--4 give the block form and the positive-diagonal condition. It remains to show that each repeated class admits non-component-wise alternatives that preserve the assumptions and the observation law. Let $C$ have common fingerprint $b$ and $|C|\ge2$, and for $Q\in O(|C|)$ let $M$ act as $\Sigma_C^{1/2}Q\Sigma_C^{-1/2}$ on the coordinates in $C$ and as the identity elsewhere. Since $M$ preserves $\Sigma_C$ and the block $D_C(s)=e^{b^\top s}\Sigma_C$ is a scalar multiple of $\Sigma_C$, $MD(s)M^\top=D(s)$ at every strength. Transforming the truth by $M$ as in Lemma~\ref{lem:homo} therefore gives a model with means $Mm(i)$, the same conditional covariances and the same grouped-view law. This model satisfies Equation~\ref{eq:response} with the same $B$ and $b_0$; its means have support $M\mathcal M$, which still affinely spans $\R^n$, and the conditional support given each auxiliary value becomes $M\mathcal M\times\mathcal S$, so Assumption~\ref{ass:ivae}(ii) and the support clause of Assumption~\ref{ass:ivae}(iii) hold with $M\mathcal M$ in place of $\mathcal M$. The decoder $f\circ(M\oplus\mathrm{id}_\beta)^{-1}$ is smooth and injective and the style laws are untouched, so Assumption~\ref{ass:ivae}(i) and the overlap clause of Assumption~\ref{ass:ivae}(iii) also hold. By the proof of Lemma~\ref{lem:homo}, $M$ is not component-wise when $Q$ is not a signed permutation.
\end{proof}

%% file: appendices/02_objective.tex
\section{Objective analysis}
\label{app:estimation}

\subsection{Alignment pressure and few-view variances}
\label{app:alignment}
Squared-distance view alignment penalizes within-instance dispersion. If $\hat\Sigma_i$ is the unbiased covariance of $R\ge2$ encoded views, then
\begin{equation}
\sum_{r<r'}\|\hat\alpha_{ir}-\hat\alpha_{ir'}\|^2=R(R-1)\,\tr\hat\Sigma_i.
\label{eq:identity}
\end{equation}
The alignment terms of contrastive and non-contrastive objectives \citep{wang2020understanding,bardes2022vicreg} therefore push each instance's view dispersion toward zero. The identity describes this pressure; it does not locate the optimum of a full objective with uniformity or variance terms. The Gaussian score of Equation~\ref{eq:score} instead scores each view separately and weights its deviation from the candidate mean by the candidate's strength-dependent precision.

Separate variance estimation is hard because each instance supplies few observations of its variances. Without recorded strengths, the $2Mn$ conditional means and variances are incidental parameters \citep{neyman1948} estimated from $MRn$ scalar observations, and each per-instance variance has $R-1$ degrees of freedom. For Gaussian views, the log sample variance from $R$ views has sampling variance $\psi'((R-1)/2)$, with $\psi'$ the trigamma function \citep{bartlett1946}. At the contrastive default $R=2$ this is $\pi^2/2\approx\noiseFloorTwo$: the per-instance estimate has a standard deviation of about 2.2 in log units, more than four times the spread of the true log-variances at $\omega=0.5$ in the synthetic designs. The affine response of Equation~\ref{eq:response} replaces the $Mn$ free log-variances with $n(d+1)$ shared parameters, so every instance informs the same estimates. Appendix~\ref{app:amort} gives the two estimators and Figure~\ref{fig:synthetic}B compares them.

\subsection{The implemented variance loss}
\label{app:logspace}
The variance term of Section~\ref{sec:objective} is computed on the calibrated codes $y$ of Appendix~\ref{app:gauge}, split into two halves of $q=R/2$ views each. Write $S_i^{(h)}$ for the unbiased covariance of half $h$, $P_{ij}$ for the variance of coordinate $j$ pooled over all $R$ views, and $d_{ij}^{(h)}=\log(S_{i,jj}^{(h)}+\epsilon)-c_q-v_{ij}$, where $c_q=\psi((q-1)/2)-\log((q-1)/2)$ with $\psi$ the digamma function and $\epsilon=10^{-6}$. The loss is
\begin{equation}
\mathcal L_{\mathrm{var}}=\frac1{|\mathcal B|}\sum_{i\in\mathcal B}
\Bigl[\underbrace{\sum_j d_{ij}^{(1)}d_{ij}^{(2)}}_{\mathcal L_{\mathrm{logvar},i}}+
\eta\underbrace{\sum_{j\ne k}\frac{S_{i,jk}^{(1)}S_{i,jk}^{(2)}}{(P_{ij}+\epsilon)(P_{ik}+\epsilon)}}_{\mathcal L_{\mathrm{corr},i}}\Bigr],\qquad\eta=1.
\label{eq:logloss}
\end{equation}
Each half covariance uses its own mean and divisor $q-1$. The pooled variance uses all $R$ views and divisor $R-1$, and gradients flow through it. The floor $\epsilon$ is added inside the logarithm and inside the pooled denominators; nothing is clipped. The off-diagonal sum runs over ordered pairs, so each unordered pair counts twice. The loss sums over coordinates and averages over instances. The image runs and the nonlinear synthetic runs that train this term both use $R=8$; the synthetic runs apply the loss to uncalibrated codes (Appendix~\ref{app:synthetic}).

\paragraph{Why split the views.} Squaring a single within-instance estimate adds its sampling variance to the loss. For raw covariance matching this addition moves the optimum: if $(R-1)\hat\Sigma_i$ is Wishart with scale $\Sigma_i$ \citep{wishart1928}, then for a fixed target $\Gamma_i$
\begin{equation*}
\E\|\hat\Sigma_i-\Gamma_i\|_F^2=\|\Sigma_i-\Gamma_i\|_F^2+\frac{(\tr\Sigma_i)^2+\|\Sigma_i\|_F^2}{R-1},
\end{equation*}
which a covariance smaller than the target minimizes. Products of independent halves remove the added term; this is the device of cross-validated distance estimates \citep{diedrichsen2016crossvalidated}. For conditionally independent Gaussian views and $\epsilon=0$, the two sums in Equation~\ref{eq:logloss} have expectations
\begin{equation*}
\E\bigl[d_{ij}^{(1)}d_{ij}^{(2)}\bigr]=\bigl(\log\Sigma_{i,jj}-v_{ij}\bigr)^2,\qquad
\E\bigl[S_{i,jk}^{(1)}S_{i,jk}^{(2)}\bigr]=\Sigma_{i,jk}^2 .
\end{equation*}
The first uses $c_q$, which makes $\log S_{i,jj}^{(h)}-c_q$ unbiased for $\log\Sigma_{i,jj}$; the second needs only independent halves. In the log-variance term a single-sample square would add only the constant $\psi'((R-1)/2)$, so here the split does not prevent collapse; it gives each term expectation zero at the target, which is what makes the near-zero variance readings of Table~\ref{tab:propc} interpretable. The implemented loss is still not unbiased: the floor $\epsilon$, the pooled denominators, which depend on both halves, and non-Gaussian codes all change its expectation. Individual cross-products can be negative, and so can the reported diagnostic.

\subsection{Centered deviations leave instance offsets free}
\label{app:offsets}
The variance terms compare views only through their deviations from the instance mean, so an offset that is constant within an instance cannot change them.
\begin{lemma}[Centered deviations ignore within-instance constants]
\label{lem:offsets}
Let $e$ be an encoder and $g$ a measurable function of the view that takes the same value on every view of an instance. For every finite set $V$ of views of an instance, $e$ and $e+g$ have the same centered deviations $e(x_{ir})-\frac1{|V|}\sum_{r'\in V}e(x_{ir'})$, in raw coordinates and under any common fixed affine calibration. Any loss that depends on the codes only through these deviations takes the same value for both encoders.
\end{lemma}
\begin{proof}
The added term is the same for every $r'\in V$, so it cancels in the centered deviation, and a common affine calibration acts on both encoders alike.
\end{proof}
Two cases of $g$ matter here. Untargeted source factors such as dSprites shape and orientation are constant within an instance and visible in every view, so any function of them can enter the content code without changing $\mathcal L_{\mathrm{var}}$. And under the shift model of Appendix~\ref{app:amounts}, if a view determines both its changed content $\at_{ir}=\alpha_i+Ja_{ir}$ and its shift $Ja_{ir}$, then $g=h(\alpha_i)-\alpha_i$ is admissible for any measurable $h$: the encoder $h(\alpha_i)+Ja_{ir}$ matches every fingerprint of the intended encoder $\at_{ir}$ while carrying arbitrary instance content, including none. The ranking loss is not blind to such offsets, because its residual against a negative candidate $u$ contains the difference $g_i-g_u$ of the two instances' offsets; this is why the variance terms are trained together with instance discrimination. The lemma holds under a fixed calibration, whereas the learner recomputes its calibration from the bank spread, which the offsets change, so two separately trained encoders are not exactly in the lemma's setting.

%% file: appendices/03_implementation.tex
\section{Image implementation and evaluation}
\label{app:impl}
Every reported image run uses the Gaussian score of Equation~\ref{eq:score} with eight views, trains its full 400-epoch budget without early stopping, and is evaluated at that final checkpoint. The studies differ only in the entries of Table~\ref{tab:response-inventory}.

\subsection{From augmentation shifts to the response model}
\label{app:amounts}
Each view applies a realized shift $a_{ir}\in\R^d$, one coordinate per augmentation family, which the learner never sees. Linearizing the renderer and the encoder in the shift,
\[
\at_{ir}=\alpha_i+J\,a_{ir},\qquad a_{ir}\mid i\sim\prod_k\mathcal N\big(0,\ \sigma_{0k}^2e^{2s_{ik}}\big),
\]
where $\alpha_i=m(i)$ is the instance's source content, $\sigma_{0k}$ is family $k$'s shift standard deviation at zero strength, and $J\in\R^{n\times d}$ collects the factor responses. A translation adds to position and a zoom adds to log-scale, so for these the relation is exact. Theorem~\ref{thm:fingerprint} conditions on the recorded strength, not on the realized shift: its observation law integrates $a_{ir}$ out, and the shift enters as the within-instance content noise with
\[
\Cov(\at_{ir}\mid i)=J\operatorname{diag}\big(\sigma_{0k}^2e^{2s_{ik}}\big)J^\top,\qquad
\Var(\at_{ir,j}\mid i)=\textstyle\sum_kJ_{jk}^2\,\sigma_{0k}^2e^{2s_{ik}}.
\]

\paragraph{Conditions for the prescribed response.} Two conditions on $J$ give the diagonal, log-affine response of Equation~\ref{eq:response}. The covariance is diagonal when every column of $J$ has at most one nonzero entry, that is, when no family changes two coordinates; a family that does induces conditional correlations. The log-variance of coordinate $j$ is affine in $s_i$ when row $j$ has exactly one nonzero entry, in column $k(j)$: then $\log\Var(\at_{ir,j}\mid i)=b_{0j}+2s_{ik(j)}$, a fingerprint of slope~2 in that family's strength, whatever affine rescaling the encoder applies to the coordinate. With two nonzero entries the log-variance is not affine on any open strength set: for $J=[1\;\;1]$ and unit magnitudes it is $\log(e^{2s_{i1}}+e^{2s_{i2}})$. Under both conditions distinct coordinates use distinct families, so the fingerprints $2e_{k(j)}$ are pairwise distinct and nonzero and Theorem~\ref{thm:fingerprint}(ii) applies. One independently controlled family per target coordinate is the intended construction and motivates the declared $B=2I_3$ on dSprites. The implemented dSprites zoom, however, is applied about the image center rather than the sprite, so it also translates an off-center sprite by an amount proportional to its position: the zoom column of $J$ acquires instance-dependent position entries, and the position covariance is then neither exactly diagonal nor exactly log-affine. The declared $B=2I_3$ is the ideal response of the design, not an exact property of the implemented augmentation.

\paragraph{What the image implementation does not satisfy.}
\label{app:limits}
The image experiments test the design, not the theorem's hypotheses, and depart from them in five ways. First, the content law can be degenerate: with $d<n$ families the covariance above has rank at most $d$, which violates the nondegeneracy in Assumption~\ref{ass:ivae}(i). On Shapes3D six coordinates answer to three hue strengths, and a smooth injective two-coordinate encoding of a continuously perturbed hue has a curve-supported conditional law even when its covariance is full rank. This restricts the shift construction, not the theorem: in the worked example of Section~\ref{sec:obstruction}, one strength scales two independent jitter draws. Second, source factors that no family targets, such as dSprites shape and orientation, are constant within an instance and vary across instances; they are neither nondegenerate content nor style whose law depends only on $\tau$, so the image model does not satisfy Equation~\ref{eq:shared}. Third, nonlinear encoder responses, clipping, resampling, finite magnitudes and pixel artifacts can all break the ideal variance law; the slope-2 target is a modeling choice for the implementation, not a verified property of the trained representation. Fourth, the Shapes3D region masks are fixed soft geometric masks, a central disc for the object and lower and upper bands for floor and wall, computed from pixel coordinates without segmentation labels; they approximate the benchmark layout rather than intervening on exact regions. Fifth, we do not test whether an augmentation acts only through the factor it targets, that is, whether the view depends on the shift only through the content the shift changed.

\subsection{Data, encoder and optimization}
\label{app:data}
\paragraph{Data.} dSprites \citep{matthey2017dsprites}: 737k binary $64\times64$ sprites with shape, scale, orientation and position. Shapes3D \citep{burgess2018shapes3d}: 480k $64\times64$ renders of a room with floor, wall and object hue, object scale, shape and camera orientation. A seeded subsample of 20,000 instances of each is cached at $64\times64$. Semantic factor values are evaluation targets only.

\paragraph{Encoder.} Every run uses the same convolutional encoder: four $4\times4$ stride-2 convolutions with 64, 128, 256 and 256 channels, each followed by a leaky ReLU of slope 0.2 and no normalization layer, then a 256-unit hidden layer with the same activation and two linear heads. The content head has $n$ outputs. The second head, called the style head in the code, has $n_\beta$ outputs, three on dSprites and two on Shapes3D; no training loss reads it and no reported metric uses it.

\paragraph{Optimization.} Adam \citep{kingma2015adam} with learning rate $10^{-3}$ and no weight decay, batches of 256 instances, ten epochs of linear warm-up and then cosine decay to 1\% of the base rate over the 400-epoch budget. The variance weight $\gamma$ ramps linearly from zero to one over the first forty epochs.

\paragraph{What the learner receives.} Each run receives images grouped by source instance and one strength vector $s_i\sim\mathcal N(0,\omega^2I_d)$ per instance, drawn once into a seeded table shared across optimization seeds (Algorithm~\ref{alg:step}, steps 1 and 3). The ranking score and the variance term read only $s_i$: $d$ numbers per instance and none per view. No objective reads the realized shifts or the semantic factors.

\paragraph{Response configuration.}
Table~\ref{tab:response-inventory} lists the coupling $B$ each study uses. A \emph{declared} coupling fixes $B$ at the design's value: $2I_3$ on dSprites, and on Shapes3D the $6\times3$ matrix whose rows assign two coordinates to each hue. A \emph{dense} coupling allocates the full $n\times d$ matrix, initialized at zero and free in every entry, with no support pattern, sparsity mask or sign constraint. In every configuration the log-variance intercepts $b_0$ are free parameters, initialized from the untrained encoder's calibrated within-instance dispersion (floored at $10^{-8}$ before the logarithm); the held-intercept controls of Table~\ref{tab:gauge} are the exception. For the learned-response runs, the fitted $B$ is saved beside the evaluation codes and is the matrix Appendix~\ref{app:learned-response} reads.

\begin{table}[htbp]
\centering\small\setlength{\tabcolsep}{3pt}
\caption{Response information supplied to each image study. $n/n_\beta$ gives the content and unused-head widths. Base amplitudes are shift standard deviations at $s=0$; translations are in normalized image coordinates, in which the image spans $[-1,1]$, so 0.15 is 7.5\% of the image side. All studies share the encoder, view count, budget and intercept initialization.}\label{tab:response-inventory}
\begin{tabular}{@{}>{\raggedright\arraybackslash}p{3.4cm}>{\raggedright\arraybackslash}p{0.9cm}>{\raggedright\arraybackslash}p{4.4cm}>{\raggedright\arraybackslash}p{3.0cm}>{\raggedright\arraybackslash}p{0.8cm}@{}}
\toprule
Study & $n/n_\beta$ & Coupling $B$ & Base amplitude & Seeds \\
\midrule
Prescribed-response dSprites (Figure~\ref{fig:dsprites-controls}) & 3/3 & Declared $2I_3$; columns $t_x$, $t_y$, log-zoom & 0.15 (each translation), 0.2 (log-zoom) & 3 \\
Learned-response dSprites (Figure~\ref{fig:images}) & 3/3 & Dense $3\times3$, learned from zero & as above & 5 \\
Shapes3D hue pairs (Table~\ref{tab:holdout}) & 6/2 & Declared rows $2e_1,2e_1,2e_2,2e_2,2e_3,2e_3$ & 0.05 turns per hue & 6 \\
\bottomrule
\end{tabular}
\end{table}

\subsection{Memory bank, calibration and the training step}
\label{app:gauge}
At initialization and at every epoch we encode eight newly sampled augmented views of every training instance and store the instance means $m_u^{\mathrm{raw}}$; these refresh views come from the augmenter's seeded stream and are not the fixed evaluation views. At the same refresh we measure the unbiased standard deviation $\varsigma_j$ of the raw means across instances, floored at $10^{-12}$, and set
\begin{equation*}
y_{ir}=D^{-1}\hat\alpha_{ir},\qquad m_u=D^{-1}m_u^{\mathrm{raw}},\qquad D=\operatorname{diag}(\varsigma_1,\ldots,\varsigma_n),
\end{equation*}
so the stored coordinates have unit empirical spread at the refresh. $D$ is detached: the gradient cannot move the current calibration, but the encoder can change the spread measured at the next refresh. After each minibatch, the batch's stored means receive an exponential moving average update with weight 0.3 on the new batch mean, in the current calibrated units \citep{wu2018instance}.

Rescaling a code coordinate moves its log-variance intercept but not its fingerprint. The ranking score has an exact unit symmetry: the joint change $y_j,m_{uj}\mapsto a_j(y_j,m_{uj})$ with $b_{0j}\mapsto b_{0j}+2\log|a_j|$ subtracts $\sum_j\log|a_j|$ from every candidate's score, which the softmax cancels. Rescaling the codes without adjusting $b_0$ is not such a symmetry. We carry $b_0$ across refreshes unchanged, so training must absorb any change in the measured spread. Table~\ref{tab:gauge} measures, in this procedure, the effect of turning the normalization off and of holding the intercepts at their initialization.

\paragraph{Normalization and intercept controls.} Table~\ref{tab:gauge} varies the two calibration choices at $\omega=0.5$: the diagonal normalization on or off, and the intercepts $b_0$ free or held. The held-intercept rows keep $b_0$ at the initialization the free rows start from, the untrained encoder's within-instance log-variance in that row's own units, as a constant rather than a trained parameter. Turning the normalization off or holding the intercepts reduces alignment and block prediction on dSprites, and the unnormalized Shapes3D runs weaken the hue grouping (Appendix~\ref{app:shapes3d}). These runs are evaluated on the training pool only.

\input{tables/gauge_ablation}

\begin{algorithm}[tb]\small
\caption{Training with recorded strengths: augmentation, bank refresh, and one step of the Gaussian image objective.}
\label{alg:step}
\begin{enumerate}[nosep,leftmargin=2em]
\item[1.] \textbf{Once.} For each instance $i$ draw $s_i\sim\mathcal N(0,\omega^2I_d)$ with one seeded generator and store the strength table.
\item[2.] \textbf{Once per epoch.} Encode eight newly sampled views of every instance and store the instance means $m_u^{\mathrm{raw}}$; set $\varsigma_j$ to the standard deviation of $m^{\mathrm{raw}}_{uj}$ across instances, floored at $10^{-12}$, and divide the stored means by $D=\operatorname{diag}(\varsigma)$. At initialization also set $b_{0j}$ to the log of the calibrated within-instance variance of coordinate $j$, floored at $10^{-8}$.
\item[3.] \textbf{Each step.} Take a batch of 256 instances. For each, form $R=8$ views: for every view and family $k$, draw the family's shift with standard deviation $\sigma_{0k}\,e^{s_{ik}}$ and apply it. Encode every view to $\hat\alpha_{ir}$ and divide by the detached $D$.
\item[4.] \emph{Instance discrimination.} Draw one set of up to 4,096 instances uniformly without replacement from outside the batch; every view in the batch ranks its own instance against these candidates. Compute $v_u=b_0+Bs_u$, with $B$ as Table~\ref{tab:response-inventory} lists and $b_0$ free, and the score $\ell_u(y_{ir})$ of Equation~\ref{eq:score}, including its candidate-dependent log-determinant; take the cross-entropy of the softmax over candidates against the view's own instance.
\item[5.] \emph{Variance term.} Split the eight calibrated views into two halves of four and compute Equation~\ref{eq:logloss}. Its weight is $\gamma=1$ after a linear ramp from zero over the first forty epochs.
\item[6.] Step on $\mathcal L_{\mathrm{sep}}+\gamma\mathcal L_{\mathrm{var}}$ with Adam; then update the batch's bank entries by the exponential moving average with weight 0.3 on the new means.
\end{enumerate}
\end{algorithm}

\subsection{Evaluation}
\label{app:evaluation}
\paragraph{Representation, instances and targets.}
\label{app:holdout}
Every image metric reads the mean of eight evaluation views per instance. The views are sampled once with a fixed seed and encoded at the final checkpoint; no metric reads the bank. The training-pool panel uses 2,000 seeded instances from the training subsample. The held-out panel uses \holdoutN{} instances drawn once per benchmark from the raw dataset outside the training subsample, augmented under the run's own families with strengths drawn afresh from $\mathcal N(0,\omega^2)$ under a seed the training never used, and encoded with the saved weights alone. Targets are the unaugmented benchmark factors in the benchmark's own units; in particular the dSprites scale target is the benchmark's scale value, not the log-zoom in which the augmentation is additive.

\paragraph{Metrics.}
Factor MCC is the Hungarian-matched mean absolute Pearson correlation between the content coordinates and the continuous factors: horizontal position, vertical position and scale on dSprites; the three region hues, object scale and camera orientation on Shapes3D, where the matching assigns each of the five factors to a distinct coordinate and the sixth coordinate does not enter. Block $R^2$ is five-neighbor regression of each factor from the whole content code, with feature standardization fitted on a random half of the evaluation instances, scored on the other half and averaged over factors. The hue-from-pair readout of Appendix~\ref{app:hue-criteria} uses the same split.

\paragraph{Why two metrics.}
\label{app:metrics}
Whole-vector prediction cannot distinguish a representation from its mixtures. For $T\in\GL(n)$, $\sigma(T\hat\alpha)=\sigma(\hat\alpha)$, so $\E[\mu\mid T\hat\alpha]=\E[\mu\mid\hat\alpha]$ almost surely for every square-integrable factor $\mu$. Any metric computed from the conditional law of $\mu$ given the representation, such as the population $R^2$ of the best predictor, therefore takes the same value on $\hat\alpha$ and on $T\hat\alpha$, and exact least-squares fits with an intercept share the invariance because $T$ preserves the column space they project onto \citep[Table~1]{kornblith2019cka}. MCC with Hungarian matching is not invariant: rotating two independent standard factors by $45^\circ$ reduces every pairwise correlation to $1/\sqrt2$. Regularized, minimum-norm and nearest-neighbor probes need not share the invariance either, so we report a probe together with an alignment metric and read neither alone as identification.

\paragraph{Mixing references.}
\label{app:references}
For each saved view-mean representation $A$ we form 2,000 draws $AQ$, with $Q$ the orthogonal factor of the QR decomposition of a matrix of independent standard Gaussian entries; column signs do not affect absolute-correlation MCC, so these give the same reference as Haar rotations. We also form draws $AG$ with $G$ of independent standard Gaussian entries, invertible almost surely and neither conditioned on a condition number nor whitened; the counts reported in the tables and text use the rotation draws. Mixing acts on the saved raw coordinates before Pearson correlation centers and standardizes columns. Factor references correlate the mixed code with the selected factors; cross-seed references correlate it with a second seed's saved means. For each run, $\hat p$ is the fraction of draws whose MCC is at least the observed value. A run exceeds its reference when $\hat p<0.01$, and every count in the tables and text uses this rule. $\hat p$ is descriptive rather than a randomization test of recovery: trained coordinates are not exchangeable with the draws, and $\hat p=0$ means only that no draw reached the observed value, at a count resolution of $1/2000$.

%% file: tables/gauge_ablation.tex
\begingroup\small\setlength{\tabcolsep}{3pt}
\begin{longtable}{@{}lllrr@{}}
\caption{Normalization and intercept controls at $\omega=0.5$: mean [minimum, maximum] over the seeds of each condition. Shapes3D has free-intercept conditions only.}\label{tab:gauge}\\
\toprule
data & norm. & intercept & MCC & kNN $R^2$ \\
\midrule\endfirsthead
\multicolumn{5}{l}{\tablename\ \thetable\ (continued)}\\
\toprule
data & norm. & intercept & MCC & kNN $R^2$ \\
\midrule\endhead
\bottomrule\endfoot
dsprites & on & free & \shortstack{0.935\\{}[0.933, 0.938]} & \shortstack{0.858\\{}[0.855, 0.862]} \\
dsprites & on & held & \shortstack{0.391\\{}[0.275, 0.590]} & \shortstack{0.255\\{}[0.083, 0.563]} \\
dsprites & off & free & \shortstack{0.259\\{}[0.123, 0.488]} & \shortstack{0.191\\{}[-0.025, 0.545]} \\
dsprites & off & held & \shortstack{0.298\\{}[0.285, 0.318]} & \shortstack{0.095\\{}[0.059, 0.128]} \\
shapes3d & on & free & \shortstack{0.397\\{}[0.305, 0.452]} & \shortstack{0.295\\{}[0.250, 0.351]} \\
shapes3d & off & free & \shortstack{0.317\\{}[0.296, 0.336]} & \shortstack{0.201\\{}[0.186, 0.220]} \\
\end{longtable}\endgroup

%% file: appendices/04_experiments.tex
\section{Experimental details and additional results}
\label{app:synthetic}

\subsection{Synthetic experiments}
\label{app:synthetic-details}
\paragraph{Data.} The recovery studies draw instance means $m(i)\sim\mathcal N(0,I_4)$ independently of strengths $s_i\sim\mathcal N(0,\omega^2I_d)$, then sample conditionally independent Gaussian views with log-variances $Bs_i$ and zero intercept. Linear observations are $x_{ir}=A\alpha_{ir}$, where $A$ has independent standard Gaussian entries and is redrawn until its condition number is at most 10. The scalar designs use $B=(1,1,1,1)^\top$ or $B=(0.5,0.8333,1.1667,1.5)^\top$; the vector designs use $B=I_4$ or a fixed matrix with independent $\mathcal N(0,1/4)$ entries drawn with seed 7. Each of the five seeds redraws the data and $A$. A single strength with distinct nonzero slopes gives pairwise distinct nonzero fingerprints, so the natural-parameter rank condition of Theorem~\ref{thm:fingerprint}(i) holds even though the log-variances $Bs_i$ have covariance $B\Cov(s)B^\top$ of rank one. Figure~\ref{fig:synthetic}A uses $M=2000$ instances, $R=4$ views and strength variation $\omega\in\{0.05,0.1,0.2,0.35,0.5,1\}$ for the vector designs, $\{0.2,0.5,1\}$ for the shared strength with distinct slopes and $\omega=1$ for equal slopes. Figure~\ref{fig:synthetic}B uses $B=I_4$, $\omega=0.5$ and $R\in\{2,3,4,8,16\}$. Amari distance is the normalized index of \citet{amari1996}; its chance level, the mean over 2,000 Haar-random rotations, is 0.42 at $n=4$, and an unidentified estimator need not sit at that line.

\paragraph{Separate and pooled variance estimators.}
\label{app:amort}
Both linear estimators remove the per-instance means with $R-1$ orthonormal view contrasts, the error contrasts of restricted likelihood \citep{patterson1971reml}, so two views suffice to define the objective. For an unmixing matrix $W$ with rows $w_j$, write $Q_{ij}(W)=\sum_r\{w_j^\top(x_{ir}-\bar x_i)\}^2$. Given log-variances $v_{ij}$, both minimize, up to constants,
\[
F(W,v)=-\log|\det W|+\frac1{2M}\sum_{i,j}\Big[v_{ij}+\frac{Q_{ij}(W)\,e^{-v_{ij}}}{R-1}\Big].
\]
The separate estimator leaves all $Mn$ values $v_{ij}$ free, which profiles them to $\log\{Q_{ij}(W)/(R-1)\}$; this is the block-Gaussian likelihood of \citet{pham2001blind} with one block per instance. The pooled estimator sets $v_{ij}=b_{0j}+b_j^\top s_i$ and fits the $n(d+1)$ response parameters jointly with $W$. Both are fit by L-BFGS-B, at most 3,000 iterations per start, from a whitening start and two random-rotation restarts, keeping the best objective. For $R\le n$ the profiled objective is unbounded below: a row of $W$ orthogonal to one instance's contrasts sends that instance's $\log Q_{ij}$ to $-\infty$ while $\log|\det W|$ stays bounded, so the reported separate-estimator values at $R\le4$ are the local minima this procedure finds. Pooling removes the degeneracy because no single instance controls its own variance. At $R=2$ the separate estimator is worse than chance on average, 0.48 against 0.42, whereas pooling reaches Amari error 0.030 (Figure~\ref{fig:synthetic}B). %

\paragraph{Pooling requires informative strengths.} Pooling helps only when the recorded strengths explain the variance variation. Table~\ref{tab:amort_miss} holds the total log-variance dispersion at 0.25 and varies how much of it the strengths explain: all (T1), none (T2) or half (T3). Pooling improves on separate estimation in T1 and T3. In T2 it is worse in every seed: uninformative strengths leave the pooled model with nearly constant variances, the ambiguous model of Lemma~\ref{lem:homo}. Recording strengths is therefore not sufficient by itself; the response model must capture the variance variation.

\begin{table}[tb]\caption{Three truth constructions with matched total log-variance dispersion (0.25) at $R=4$, $M=2000$ and five seeds; entries are mean Amari error, lower being better. T1 assigns all log-variance variation to the recorded response, T2 assigns it to unrecorded variation and T3 splits it equally; the unrecorded standard deviations are 0, 0.500 and 0.354. The separate estimator (column ``plug'') sees the same log-variance distribution in all three truths; its row differences come from single seeds near chance in T2 and T3. The pooled estimator is column ``D''.}\label{tab:amort_miss}\begin{center}\small\input{tables/revision_amort_miss}\end{center}\end{table}

\paragraph{Nonlinear data.} Content follows the same Gaussian construction with $B=I_4$. Independently for each view, $\tau_{ir}\sim\mathcal N(0,I_2)$ and $\beta_{ir}\mid\tau_{ir}\sim\mathcal N(\tau_{ir},0.3^2I_2)$, independently of content. The decoder acts on the six concatenated coordinates through three affine layers with leaky-ReLU activations of slope 0.2 between them; each weight matrix has independent $\mathcal N(0,1/6)$ entries, redrawn until its condition number is at most 5, and biases have independent $\mathcal N(0,0.1^2)$ entries. The decoder, means and strengths are drawn once and shared by the training seeds, which differ in encoder initialization and view draws; content and style views are resampled each epoch.

\paragraph{Nonlinear learner.} The encoder is a three-layer MLP of width 128 with leaky ReLU of slope 0.2 and two linear heads, four content and two style outputs. The Gaussian score receives the true $B$ and learns its intercepts, except in the learned-response study below, which fits $B$ as well; the variance term of Equation~\ref{eq:logloss} is applied to the raw codes with $R=8$ and adds its own free intercept. Unlike the image learner, the bank stores raw instance means updated by an exponential moving average with weight 0.3 on the new mean, without per-epoch refresh or calibration, and every out-of-batch instance is a negative candidate. The style head regresses the recorded per-view auxiliary $\tau_{ir}$ by squared error with unit weight, so this learner reads $\tau$, which the theorem's observation law also contains. Training uses Adam at $10^{-3}$ with batches of 250 instances, 100 warm-up epochs and then cosine decay to 1\%, for the full 6,000 epochs with $\gamma$ held constant at 0 or 1; the learned-response study and its true-$B$ reference instead use four views, $\gamma=0$ and early stopping, as stated below.

\paragraph{Nonlinear results.} Table~\ref{tab:propc} reports the Gaussian score with and without the variance term at $\omega=0$ and $0.5$. View MCC scores individual encoded views against their latent realizations, since the encoder estimates $\alpha_{ir}$. Mean block $R^2$ regresses the true means $m(i)$ on the eight-view code mean; the same regression from the true content mean gives its ceiling, 0.888 at $\omega=0$ and 0.877 at $\omega=0.5$, so every cell sits at its ceiling and the lower value at $\omega=0.5$ reflects the larger within-instance variance, not lost information. Amari distance is computed from the least-squares linear map between the true means and the code means. With strength variation both variants reach near-exact alignment; without it neither does. The ranking score already conditions on the strengths, so removing the explicit variance term does not remove the response information.

\keeptogether{16}
\input{tables/propc}

\paragraph{Learned dense response.}
Table~\ref{tab:synthdense} withholds $B$ from the Gaussian score and replaces it with a dense $4\times4$ matrix, initialized at zero and free in every entry, with no support pattern, so the learner is not told which strength drives which coordinate; this is the synthetic counterpart of the learned-response dSprites sweep of Appendix~\ref{app:learned-response}. The intercepts $b_0$ are learned as in every Gaussian configuration, and $B$ and $b_0$ train with the encoder at the same Adam rate. The data are those of the nonlinear paragraphs above, with $B=I_4$ and slope 1. The configuration differs from Table~\ref{tab:propc} in three ways: four views per instance instead of eight, no variance term ($\gamma=0$ throughout), and early stopping on a training-loss plateau (the last twelve readings of the ranking loss, taken every 100 epochs, spread and drift within 1\% of their scale, after at least 500 epochs), which ended the runs between 1,500 and 3,000 epochs. The sweep covers $\omega\in\{0,0.2,0.5,1\}$ with five training seeds each. Its reference is a run of the identical configuration whose score receives the true $B$, three seeds at $\omega=0.5$, the row marked true; the dense rows should be read against that row and not against Table~\ref{tab:propc}, whose rows use eight views and the full budget.

Without strength variation, alignment fails as it does in Table~\ref{tab:propc}: view MCC is \synthDenseMccOff{}. There every $s_i=0$, so $b_0+Bs_i$ does not depend on $B$, which receives no gradient and stays exactly zero; the fitted-$B$ readouts are undefined at that strength. At $\omega=0.2$, $0.5$ and $1$, view MCC is \synthDenseMccTwo{}, \synthDenseMccHalf{} and \synthDenseMccOne{}, with no seed below \synthDenseMccOnLo{}. The true-$B$ reference reaches \synthOracleMccHalf{} at $\omega=0.5$, so withholding $B$ does not change alignment in this configuration. Mean block $R^2$ sits at the ceilings stated above at $\omega=0$ and $0.5$; it is lower at $\omega=1$, where no ceiling has been computed.

We read the fitted $B$ with the three readouts that Appendix~\ref{app:learned-response} defines for the dSprites sweep, with strengths in place of augmentation families: concentration $C(B)$, the share of the matrix's absolute mass that a one-to-one matching of coordinates to strengths collects; the matched slopes, the entries $|B_{j,\pi(j)}|$ that matching selects; and agreement, which holds for a seed when, for every coordinate, the strength $B$ assigns it to is the strength that drives, in the true $B$, the latent content coordinate the encoded coordinate aligns with under a Hungarian matching of absolute correlations. At every $\omega>0$ concentration is at least \synthDenseShareLo{}, the matched slopes lie between \synthDenseSlopeLo{} and \synthDenseSlopeHi{} against the true slope 1, and all \synthDenseAgree{} of \synthDenseAgreeN{} seeds agree. The learner therefore recovers not only the coordinates but the assignment of strengths to coordinates, as a near-permutation matrix with slopes near the true value, without the variance term. These are descriptive readouts of one coupling, $B=I_4$; no other coupling was run in this configuration.

\begin{table}[tb]
\centering\scriptsize\setlength{\tabcolsep}{2.5pt}
\caption{Nonlinear synthetic recovery with the coupling withheld. The Gaussian score learns a dense $4\times4$ $B$ from zero (rows marked dense, five seeds) or receives the true $B=I_4$ (the row marked true, three seeds), in both cases with four views, $\gamma=0$ and early stopping. Entries are the seed mean [minimum, maximum]. Concentration, matched slopes and agreement read the fitted $B$ as in Table~\ref{tab:densefit}, with strengths in place of augmentation families and true slope 1; matched slopes span the selected entries of all seeds, and agreement counts the seeds whose fitted assignment matches their latent alignment for all four coordinates. These readouts are undefined at $\omega=0$, where $B$ receives no gradient and stays zero, and do not apply in the true row, whose $B$ is supplied rather than fitted.}
\label{tab:synthdense}
\begin{tabular}{@{}lcc rrr rrc@{}}\toprule
Setting & $\omega$ & Seeds & View MCC & \shortstack{Mean block\\$R^2$} & Amari & \shortstack{Concentration\\$C(B)$} & \shortstack{Matched\\slopes} & Agreement \\\midrule
dense & 0 & 5 & 0.778 [0.659, 0.868] & 0.887 [0.887, 0.887] & 0.407 [0.318, 0.546] & undefined ($B=0$) & -- & -- \\
dense & 0.2 & 5 & 0.997 [0.996, 0.997] & 0.885 [0.885, 0.886] & 0.017 [0.011, 0.021] & 0.97 [0.97, 0.98] & 0.97 to 1.02 & 5/5 \\
dense & 0.5 & 5 & 0.997 [0.996, 0.998] & 0.876 [0.875, 0.876] & 0.018 [0.009, 0.026] & 0.99 [0.98, 0.99] & 0.97 to 1.02 & 5/5 \\
dense & 1 & 5 & 0.998 [0.997, 0.999] & 0.833 [0.833, 0.834] & 0.020 [0.012, 0.026] & 0.99 [0.99, 0.99] & 0.96 to 1.01 & 5/5 \\
true & 0.5 & 3 & 0.995 [0.993, 0.996] & 0.875 [0.874, 0.875] & 0.025 [0.022, 0.027] & -- & -- & -- \\
\bottomrule\end{tabular}
\end{table}

\subsection{Prescribed-response dSprites and controls}
\label{app:dsprites-primary}
\label{app:images}
Table~\ref{tab:holdout} reports the prescribed-response dSprites conditions and the Shapes3D design on the training pool and on the held-out instances of Appendix~\ref{app:holdout}; Table~\ref{tab:learnedresponse} does the same for the learned-response sweep. Together they cover the \holdoutCells{} reported encoders: twelve prescribed-response dSprites, twenty learned-response dSprites and twelve Shapes3D. Across these encoders the seed-mean MCC and kNN block $R^2$ change by at most \holdoutMaxGapMcc{} and \holdoutMaxGapKnn{} between the two instance sets. The held-out instances share the benchmarks and augmentation distributions with training, so they test transfer to new instances, not to new intervention families or unseen factor combinations.

At $\omega=0.5$ with $B=2I_3$, the three seeds reach factor MCC \dsMainMccSeeds{} and mean kNN block $R^2$ \dsMainKnn{}. No rotation draw reaches any run's MCC; the runs' reference 99th percentiles lie between 0.65 and 0.81. Between seeds, MCC lies between 0.998 and 1.000, and each of the three seed pairs exceeds its cross-seed reference, whose 99th percentiles lie between 0.70 and 0.86; the pairs share trained encoders and are not independent replications. %

The controls use the same reference construction, recomputed for each run. Without strength variation, block prediction stays high (kNN $R^2$ \dsOmZeroKnn{}) while MCC falls to \dsOmZeroMcc{} and is \dsOmZeroNull{}. Permuting strengths across instances leaves the views unchanged but supplies another instance's strengths to both the score and the variance term; MCC then stays \dsPermNull{}, on both instance sets. Removing the variance term lowers MCC to \dsGzeroMcc{}, yet MCC is still \dsGzeroNull{}, and \dsGzeroXsPairs{} seed pairs exceed their cross-seed reference. Exceeding a reference therefore does not certify recovery: stable weak or unwanted features can pass these checks.

\begin{figure}[tb]
\centering
\includegraphics[width=.82\linewidth]{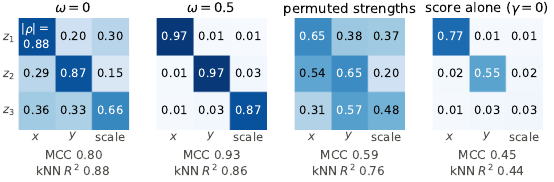}
\caption{\textbf{Prescribed-response dSprites and its controls} ($B=2I_3$, three seeds per condition), in the form of Figure~\ref{fig:images}A: absolute correlation between each learned coordinate and each targeted factor, with coordinates matched to factors in each run and entries averaged over seeds. Permuted strengths mix the factors across coordinates; the score alone leaves the scale factor unencoded.}
\label{fig:dsprites-controls}
\end{figure}

\begin{table}[tb]
\centering\footnotesize\setlength{\tabcolsep}{3pt}
\caption{Prescribed-response dSprites conditions and the Shapes3D design on the training-pool evaluation instances and on \holdoutN{} instances outside encoder training: seed mean $\pm$ sample standard deviation of factor MCC and kNN block $R^2$, and the number of runs exceeding their rotation reference ($\hat p<0.01$) on the training pool and on the held-out instances (refs.). The Shapes3D hue pairs are judged against the hue-from-pair reference of Appendix~\ref{app:hue-criteria} instead; a dash marks those rows.}
\label{tab:holdout}
\begin{tabular}{@{}lcc rr rr c@{}}\toprule
 & & & \multicolumn{2}{c}{Factor MCC} & \multicolumn{2}{c}{kNN $R^2$} & Refs. \\
Condition & $\omega$ & Seeds & training pool & held-out & training pool & held-out & \\\midrule
\multicolumn{8}{@{}l}{\textbf{dSprites}} \\
Strength variation on & 0.5 & 3 & 0.935 $\pm$ 0.003 & 0.934 $\pm$ 0.003 & 0.858 $\pm$ 0.004 & 0.855 $\pm$ 0.006 & 3/3; 3/3 \\
Strength variation off & 0 & 3 & 0.803 $\pm$ 0.066 & 0.801 $\pm$ 0.066 & 0.880 $\pm$ 0.003 & 0.879 $\pm$ 0.003 & 1/3; 1/3 \\
Permuted strengths & 0.5 & 3 & 0.592 $\pm$ 0.020 & 0.584 $\pm$ 0.020 & 0.764 $\pm$ 0.022 & 0.766 $\pm$ 0.015 & 0/3; 0/3 \\
Variance term removed & 0.5 & 3 & 0.452 $\pm$ 0.240 & 0.445 $\pm$ 0.248 & 0.439 $\pm$ 0.224 & 0.410 $\pm$ 0.227 & 3/3; 3/3 \\
\addlinespace
\multicolumn{8}{@{}l}{\textbf{Shapes3D}} \\
Hue pairs, variation on & 0.5 & 6 & 0.397 $\pm$ 0.056 & 0.390 $\pm$ 0.055 & 0.295 $\pm$ 0.039 & 0.284 $\pm$ 0.029 & -- \\
Hue pairs, variation off & 0 & 6 & 0.340 $\pm$ 0.058 & 0.335 $\pm$ 0.059 & 0.236 $\pm$ 0.047 & 0.236 $\pm$ 0.045 & -- \\
\bottomrule\end{tabular}
\end{table}

\subsection{Learned-response dSprites}
\label{app:learned-response}
The learned-response sweep replaces the declared coupling with the dense $3\times3$ matrix of Table~\ref{tab:response-inventory}, learned from zero, and otherwise matches the prescribed-response study; the coordinate-to-family assignment is withheld, the factor-targeted families are not. Moderate strength variation improves alignment: mean factor MCC is \trainDspMccTwo{} at $\omega=0.2$ and \trainDspMccHalf{} at $\omega=0.5$, against the $\omega=0$ row of Table~\ref{tab:learnedresponse}, and no rotation draw reaches the observed value in any run at these two strengths. At $\omega=1$ both alignment and block prediction fall.

\paragraph{Reading the fitted coupling.}
Alignment alone would not show that the learner found an assignment rather than tolerating its absence, so we read the fitted $B$ through three descriptive readouts per seed (Table~\ref{tab:densefit}). \emph{Concentration},
\begin{equation*}
C(B)=\frac{\max_{\pi}\sum_j|B_{j,\pi(j)}|}{\sum_{j,k}|B_{jk}|},
\end{equation*}
is the share of the matrix's absolute mass that a one-to-one matching of coordinates to families can collect; it equals one when each row and column has a single nonzero entry. The \emph{matched slopes} are the entries $|B_{j,\pi(j)}|$ selected by the maximizing matching. \emph{Agreement} compares two matchings computed independently: the maximizing $\pi$ assigns each coordinate to an augmentation family, and a Hungarian matching of the coordinates' absolute correlations with the benchmark factors assigns each coordinate to a factor. A seed agrees when, for all three coordinates, the family the coordinate answers to in $B$ is the family that shifts the factor it aligns with. At $\omega=0$ all $s_i$ vanish, $b_0+Bs_i$ does not depend on $B$, and the dense matrix stays at zero, so these readouts are undefined there.

Where alignment is high, at $\omega=0.2$ and $0.5$, concentration is high and every seed agrees: the fitted mass sits on one entry per row, the shape a declared assignment would have. The matched slopes stay below the ideal additive-shift value of two (Appendix~\ref{app:amounts}). At $\omega=1$ concentration falls with alignment and one seed no longer agrees.

\begin{table}[tb]
\centering\small
\caption{Fitted-coupling readouts of the learned-response sweep, five seeds per strength. Concentration is the seed mean [minimum, maximum]; matched slopes span the selected entries of all seeds; agreement counts the seeds whose fitted assignment matches their factor alignment for all three coordinates.}
\label{tab:densefit}
\begin{tabular}{@{}crrr@{}}\toprule
$\omega$ & Concentration $C(B)$ & Matched slopes & Agreement \\\midrule
0 & undefined ($B=0$) & -- & -- \\
0.2 & 0.78 [0.60, 0.83] & 0.82 to 1.66 & 5/5 \\
0.5 & 0.90 [0.88, 0.92] & 1.64 to 1.71 & 5/5 \\
1 & 0.58 [0.37, 0.81] & 0.05 to 0.65 & 4/5 \\
\bottomrule\end{tabular}
\end{table}

\subsection{Shapes3D hue pairs}
\label{app:shapes3d}
The grouping experiment assigns two coordinates to each of three region-wise hue rotations, because two coordinates can encode a full circle continuously and one cannot do so injectively. Under the Gaussian model this repeated-fingerprint pattern identifies three two-dimensional blocks (Theorem~\ref{thm:fingerprint}(iii)); Appendix~\ref{app:limits} states how the circular content departs from that model.

\paragraph{Hue-from-pair readout and its mixing reference.}
\label{app:hue-criteria}
Each prescribed coordinate pair predicts each hue circle $(\cos2\pi h,\sin2\pi h)$ by five-neighbor regression, with feature standardization fitted on one half of the evaluation instances and scored on the other, averaging the two $R^2$ values; this gives a three-pair by three-hue matrix per run. A Hungarian assignment maximizes the sum over pairs and hues; the \emph{assigned mean} averages the three selected entries and the \emph{off-assignment mean} the remaining six. The mixing reference mixes the six-dimensional content code of the same run and rereads it the same way: each of 200 draws forms an orthogonal $Q$ from the QR decomposition of a standard Gaussian matrix, recomputes the matrix on the mixed code with the prescribed pairs unchanged, and repeats the Hungarian assignment within the draw. The compared statistic is the assigned mean, and $\hat p$ is the fraction of draws whose assigned mean is at least the observed one; 200 draws give a resolution of $1/200$, coarser than the 2,000-draw MCC references.

\paragraph{Results.} Table~\ref{tab:s3dscores} lists every seed's assigned and off-assignment means with its matched hue order. With strength variation, every seed takes the prescribed hue order and exceeds its reference at $\hat p<0.01$, but the assigned means span a wide range, so the shared assignment is not uniformly strong recovery. Without strength variation the matched orders differ across seeds and the assigned means stay close to the off-assignment means. The three unnormalized seeds reach lower assigned means still and do not all take the prescribed order. Cross-seed reference exceedances at $\hat p<0.01$ occur in \sThreeXsOn{} of \sThreeXsPairs{} seed pairs with strength variation and \sThreeXsOff{} of \sThreeXsPairs{} without it, so cross-seed agreement alone does not establish the intended grouping. Content MCC over the five continuous factors (Appendix~\ref{app:evaluation}) changes from \trainGaussSthreeMcc{} to \holdoutGaussSthreeMcc{} on held-out instances (Table~\ref{tab:holdout}).

\keeptogether{22}
\input{tables/s3d_matrices}

%% file: tables/revision_amort_miss.tex
\begin{tabular}{lrr}
\toprule
truth & plug & D \\
\midrule
T1 recorded only & 0.0749 & 0.0122 \\
T2 unrecorded only & 0.1241 & 0.3042 \\
T3 half \& half & 0.1442 & 0.0212 \\
\bottomrule
\end{tabular}

%% file: tables/propc.tex
\begingroup\small\setlength{\tabcolsep}{3pt}
\begin{longtable}{@{}lrrrrrr@{}}
\caption{Nonlinear synthetic results for the Gaussian score with ($\gamma=1$) and without ($\gamma=0$) the variance term: mean [minimum, maximum] over three seeds at $n=4$, eight views and the full 6,000-epoch budget. The variance diagnostic uses the log-variance/correlation reduction of Equation~\ref{eq:logloss}; in the $\gamma=0$ rows it is a readout, not a trained term.}\label{tab:propc}\\
\toprule
setting & $\gamma$ & $\omega$ & view MCC & \shortstack{mean block\\$R^2$} & Amari & variance term \\
\midrule\endfirsthead
\multicolumn{7}{l}{\tablename\ \thetable\ (continued)}\\
\toprule
setting & $\gamma$ & $\omega$ & view MCC & \shortstack{mean block\\$R^2$} & Amari & variance term \\
\midrule\endhead
\bottomrule\endfoot
Gaussian & 1 & 0 & \shortstack{0.729\\{}[0.720, 0.740]} & \shortstack{0.887\\{}[0.887, 0.887]} & \shortstack{0.490\\{}[0.468, 0.504]} & \shortstack{-0.03\\{}[-0.06, -0.01]} \\
Gaussian & 1 & 0.5 & \shortstack{0.999\\{}[0.999, 0.999]} & \shortstack{0.876\\{}[0.876, 0.876]} & \shortstack{0.010\\{}[0.009, 0.011]} & \shortstack{-0.00\\{}[-0.03, 0.03]} \\
Gaussian & 0 & 0 & \shortstack{0.768\\{}[0.708, 0.843]} & \shortstack{0.887\\{}[0.887, 0.888]} & \shortstack{0.428\\{}[0.357, 0.490]} & \shortstack{0.01\\{}[-0.02, 0.05]} \\
Gaussian & 0 & 0.5 & \shortstack{1.000\\{}[1.000, 1.000]} & \shortstack{0.877\\{}[0.877, 0.877]} & \shortstack{0.008\\{}[0.008, 0.009]} & \shortstack{0.00\\{}[-0.02, 0.03]} \\
\end{longtable}\endgroup

%% file: tables/s3d_matrices.tex
\begingroup\small\setlength{\tabcolsep}{3pt}
\begin{longtable}{@{}llrrrl@{}}
\caption{Shapes3D hue-from-pair summaries for every seed, from the unrounded matrices: the Hungarian-assigned mean, the off-assignment mean and the matched hue order (the prescribed order is floor, wall, object).}\label{tab:s3dscores}\\
\toprule
normalization & $\omega$ & seed & assigned & off-assignment & matched hue order \\
\midrule\endfirsthead
\multicolumn{6}{l}{\tablename\ \thetable\ (continued)}\\
\toprule
normalization & $\omega$ & seed & assigned & off-assignment & matched hue order \\
\midrule\endhead
\bottomrule\endfoot
normalized & 0.5 & 0 & 0.31 & -0.12 & floor, wall, object \\
normalized & 0.5 & 1 & 0.64 & -0.16 & floor, wall, object \\
normalized & 0.5 & 2 & 0.67 & -0.17 & floor, wall, object \\
normalized & 0.5 & 3 & 0.38 & -0.01 & floor, wall, object \\
normalized & 0.5 & 4 & 0.67 & -0.15 & floor, wall, object \\
normalized & 0.5 & 5 & 0.36 & -0.10 & floor, wall, object \\
normalized & 0 & 0 & 0.12 & -0.01 & floor, object, wall \\
normalized & 0 & 1 & 0.28 & 0.18 & object, floor, wall \\
normalized & 0 & 2 & 0.10 & -0.01 & object, floor, wall \\
normalized & 0 & 3 & 0.12 & -0.01 & floor, object, wall \\
normalized & 0 & 4 & 0.13 & -0.01 & floor, object, wall \\
normalized & 0 & 5 & 0.08 & -0.03 & floor, wall, object \\
unnormalized & 0.5 & 0 & 0.11 & 0.05 & wall, floor, object \\
unnormalized & 0.5 & 1 & 0.20 & 0.11 & floor, wall, object \\
unnormalized & 0.5 & 2 & 0.19 & 0.07 & floor, wall, object \\
\end{longtable}\endgroup